\documentclass[sigconf]{mobihoc26_style/acmart}
\usepackage{amsmath}  
\usepackage{xcolor}
\usepackage{float}
\usepackage{caption} 
\usepackage{url}
\usepackage{hyperref}
\usepackage{bm}
\usepackage{graphicx}
\usepackage[ruled,vlined,linesnumbered]{algorithm2e}
\usepackage{subcaption}
\usepackage{wrapfig}
\usepackage{fancyhdr}
\usepackage{booktabs}
\usepackage{multicol}
\usepackage{setspace}
\usepackage{style}

\newcommand{\x}{\bm{x}}
\newcommand{\bu}{\bm{u}}
\newcommand{\z}{\bm{z}}
\newcommand{\fedopi}{{\tt FedOAG}}
\AtBeginDocument{%
  }

\setcopyright{acmlicensed}
\copyrightyear{2026}
\acmYear{2026}
\acmDOI{XXXXXXX.XXXXXXX}
\acmConference[Conference acronym 'XX]{Make sure to enter the correct
  conference title from your rights confirmation email}{June 03--05,
  2026}{Woodstock, NY}
\acmISBN{978-1-4503-XXXX-X/2018/06}

\begin{document}

\title{Over-the-Air Federated Learning in Heterogeneous Mobile Wireless Networks}

\author{Ming Xiang$^1$, 
Nicolò Michelusi$^2$, 
Yonina C. Eldar$^{1,3}$, 
Lili Su$^1$}
\affiliation{%
  \institution{$^1$Northeastern University, Boston, MA, USA; 
  $^2$Arizona State University, Tempe, AZ, USA;
  $^3$Weizmann Institute of Science, Rehovot, Israel.
  }
    \country{}
}

\renewcommand{\shortauthors}{Xiang et al.}

\begin{abstract}
  Over-the-air computation has emerged as a scalable and efficient solution for deploying federated learning algorithms in wireless networks by exploiting waveform superposition for simultaneous model aggregation.
  Most existing work struggles with heterogeneous fading channels.
  These approaches either enforce unbiased updates from all devices or 
  allow partial device contributions,
  requiring careful tuning of the convergence bound to mitigate bias under specific fading models.
  However, the former significantly amplifies receiver noise due to the weakest channel, whereas the latter is sensitive to fading model mismatch and converges only to a biased objective.
  To tackle these challenges, 
  we propose~\fedopi, which employs algorithmic components to 
   automatically satisfy energy constraints via gradient normalization and
  evenly mix devices' updates through implicit gossiping.
  Importantly,~\fedopi~does not require transmission from all devices,
  nor does it rely on a specific fading model or 
  knowledge of time-varying statistical channel distributions.
  We show that \fedopi~converges to a stationary point of an unbiased non-convex objective at the best possible rate $\calO(1/\sqrt{T})$ 
  for any stochastic first-order method.
  We corroborate our analysis with numerical experiments over dynamic wireless conditions on real-world datasets.
\end{abstract}



\received{20 February 2007}
\received[revised]{12 March 2009}
\received[accepted]{5 June 2009}

\maketitle

\section{Introduction}
\label{sec: introduction}
Federated learning is a distributed machine learning framework that enables distributed devices to collectively train a global model \cite{mcmahan2017communication}.
Instead of sending raw data directly to the parameter server, devices process their data locally and periodically report updates to it.
In wireless networks, the high-dimensional device updates are transmitted over noisy and bandwidth-limited wireless channels, creating a communication bottleneck that limits scalability \cite{3gpp38901,iturm1225}.

Over-the-air (OTA) federated learning addresses this by exploiting the principle of waveform superposition on the multiple access channel (MAC), enabling a noisy single-shot aggregation \cite{yang2020federated,tegin2023federated,zhu2024over}.
The single-shot feature renders conventional federated learning approaches largely infeasible, as the parameter server receives only a superimposed analog waveform and cannot isolate, decode, or individually reweight the updates from distinct devices.
The OTA federated learning approach, by contrast, focuses on designing power-control strategies to achieve high-quality global aggregation under an energy constraint. 
For example, the seminal OTA federated learning algorithm \cite{yang2020federated} ensures unbiased global aggregation over a fading channel by scaling down the device signal based on the weakest device channel; however, this aggressive down-scaling forces the parameter server to massively amplify received transmission, inadvertently blowing up the receiver noise.

Truncated channel inversion through thresholding has then been adopted to facilitate partial device contributions and exclude devices with poor channels to mitigate noise amplification at the receiver
\cite{amiri2020federated,zhu2019broadband}, yet with an implicit homogeneous wireless path loss assumption across devices \cite{amiri2020federated,sery2021over,mao2022charles}.
In reality, heterogeneous wireless channel conditions and a diversified device population are two defining features of OTA federated learning systems. 
When mobile devices are spatially distributed, their statistical channel distribution information (CDI) becomes both heterogeneous and time-varying.
Relying on homogeneous channels introduces prohibitive overhead for continuously tracking time-varying CDIs or inevitably leads to biased global aggregation.
Biased global aggregation leads to objective inconsistency and can significantly harm the federated learning performance. Details can be found in~\prettyref{sec: problem formulation}.
When path losses are heterogeneous, a line of work \cite{cao2020optimized,zhu2019broadband} relaxes unbiased aggregation in exchange for reduced noise amplification; however, the resulting bias remains uncorrected, and its impact on learning performance is not quantitatively characterized.
Recently, Abrar and Michelusi in \cite{abrar2025non,11475389} propose a successive convex approximation framework to simultaneously optimize the learning and power control by designing an optimal pre-scalar at each device.
However, the learning objective remains biased, 
the analysis therein only accounts for Rayleigh fading, and does not readily extend to other fading models.
Moreover, the statistical CDIs are assumed to be time-invariant and known a priori, which may break down under device mobility.
As devices move, channel distributions drift over time, rendering any precomputed CDIs quickly outdated.

\noindent{\bf Contributions.}
In this paper, we develop~\fedopi~
that converges to an unbiased objective over noisy MACs without requiring transmission from every device, avoiding high variance and degraded utility.
To capture heterogeneous and time-varying channel effects, we model device $i$'s uplink availability in round $t$ by probability $p_i^t$, while assuming reliable downlink communication.
An illustration plot can be found in~\prettyref{fig:heterogeneous mobile OTA setup}.

Importantly, neither the parameter server nor the devices know $p_i^t$'s or the underlying fading models.
This precludes conventional OTA approaches that rely on exact channel knowledge for power control and scheduling.
A detailed formulation of $p_i^t$'s can be found in~\prettyref{sec: problem formulation}.
\begin{figure}[!t]
    \centering
    \includegraphics[width=\linewidth,trim=6cm 5.5cm 7cm 2cm, clip]{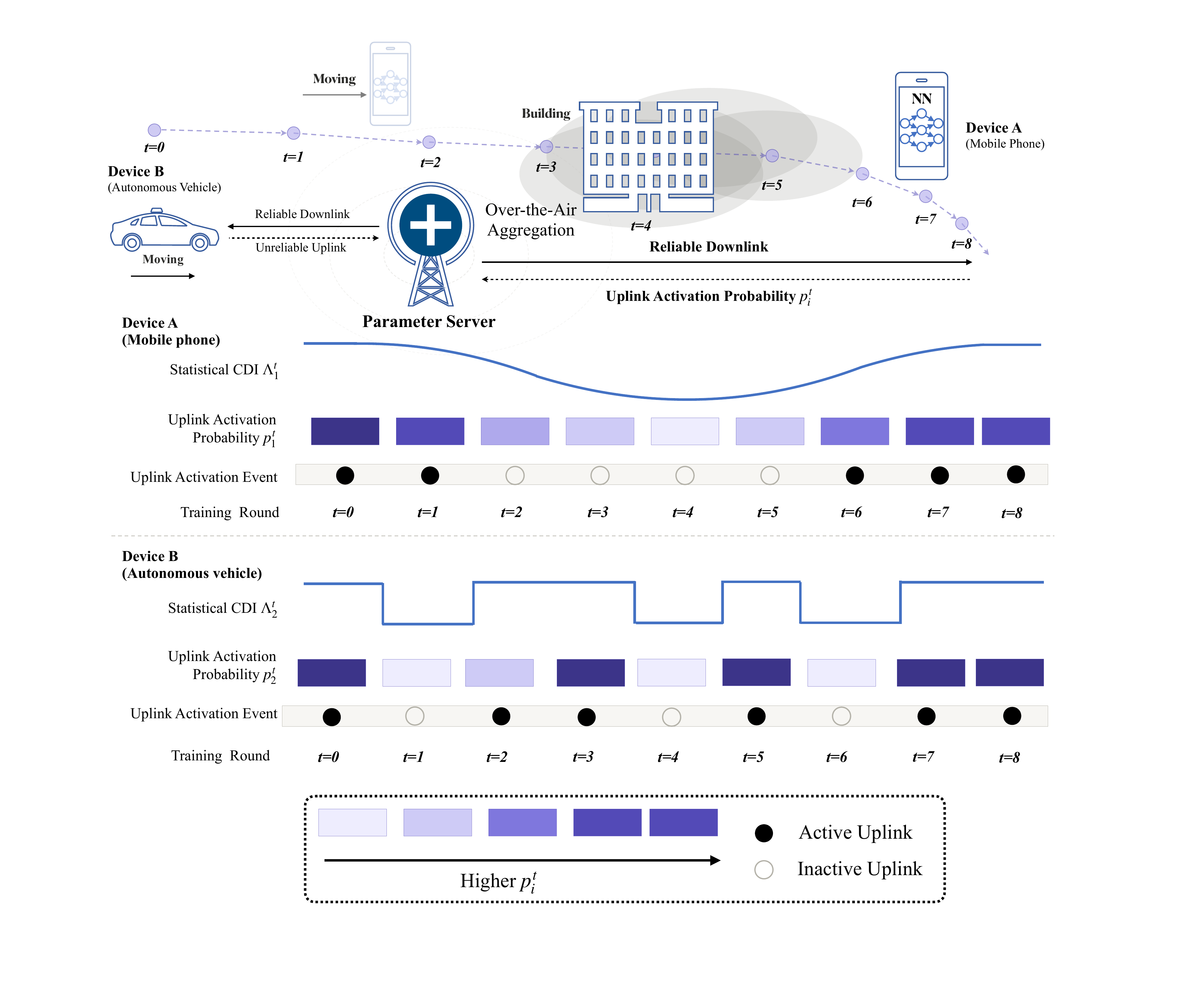}
    \caption{\small
    Mobile devices lead to time-varying statistical CDIs, which determine the uplink activation probabilities $ p_i^t$ for heterogeneous devices.}
    \label{fig:heterogeneous mobile OTA setup}
    \vspace*{-1.5\baselineskip}
\end{figure}
Our contributions are three-fold:
\begin{itemize}
    \item We propose~\fedopi.
    The design of~\fedopi~relies on two algorithmic innovations:
    (i) channel inversion with gradient normalization,
    and (ii) postponed global model broadcast.
    Generally speaking, these algorithmic components
    (i) tackle energy-constrained transmission with partial device contributions over noisy MACs,
    and (ii) enable devices to evenly mix their information through implicit device-device gossip without the knowledge of fading models or time-varying $p_i^t$'s. 
    It is worth noting that no direct device-device communication is involved.
    \item We provide a non-asymptotic convergence bound.
    Despite dynamic, heterogeneous, and mobile wireless conditions,
    we show that
    \fedopi~converges at a rate of $\calO(1/\sqrt{T})$, 
    which is the best possible rate for any first-order method that has access to only stochastic gradients.
    \item We corroborate our analysis with experiments over dynamic wireless channel conditions on real-world datasets.
\end{itemize}
\noindent{\bf Organizations.}
The remainder of this paper is organized as follows.
\prettyref{sec: related work} reviews related work.
\prettyref{sec: problem formulation} formulates the problem.
\prettyref{sec: algorithm} introduces the design of~\fedopi.
\prettyref{sec: convergence analysis} analyzes the convergence of~\fedopi.
\prettyref{sec: numerical experiment} presents numerical experiments.
\prettyref{sec: conclusion} concludes this paper.

\noindent{\bf Notations.}
We introduce the notations that we will use throughout the paper.
$\norm{\bm{v}}$ defines its $l_2$ norm for a given vector $\bm{v}$.
$\fnorm{A}$ defines the Frobenius norm of a given matrix $A$, and 
$\lambda_2(A)$ denotes its second largest eigenvalue when $A$ is a square matrix.  
$\reals^d$ defines a $d$-dimensional vector space. %
$[m] \triangleq \{1, \cdots, m\}$. %
$\Indc$ is an all-one vector.
$\indc{\calE}$ is an indicator function of event $\calE$, \ie, $\indc{\calE}=1$ when the event $\calE$ occurs; $\indc{\calE}=0$ otherwise.
$\calF^t$ denotes the sigma-algebra generated by the randomness up to round $t$. 
$f(n) = \calO( g(n))$ if there exist 
$c_0>0$ and $n_0 \in \naturals$ such that $f(n)\le c_0 g(n)$, $\forall n \ge n_0$.

\section{Related Work}
\label{sec: related work}

\begin{table}[!t]
    \centering
    \resizebox{\linewidth}{!}{
    \begin{tabular}{ccc|cc|c}
    \toprule 
    \multirow{2}{*}{\bf
    Algorithms
    }
    & \multicolumn{2}{c|}{\bf Small-scale CDIs $h_{i}^{t}$} 
    & \multicolumn{2}{c|}{\bf Statistical CDIs $\Lambda_{i}^{t}$} 
    &
    \multirow{2}{*}{\bf Fading Model}
    \\
    \cline{2-5}
    & {\bf The PS} & {\bf The Devices}
    & {\bf The PS} & {\bf The Devices}
    \\
    \hline 
    {\bf Ours (\fedopi)}
    & 
    $\times$ 
    & 
    $\checkmark$
    & 
    $\bm{\times}$
    & 
    $\bm{\times}$
    &
    {\bf Arbitrary}
    \\
    SCA
    \cite{11475389,abrar2025non}
    & 
    $\times$ 
    & 
    $\checkmark$
    & 
    $\checkmark$ 
    & 
    $\times$
    &
    Rayleigh
    \\
    Vanilla OTA
    \cite{yang2020federated}
    & 
    $\checkmark$
    & 
    $\checkmark$
    & 
    $\times$ 
    & 
    $\times$ 
    &
    Arbitrary
    \\
    BB methods 
    \cite{zhu2019broadband}
    & 
    $\times$ 
    & 
    $\checkmark$
    & 
    $\checkmark$ 
    & 
    $\checkmark$ 
    &
    Rayleigh
    \\
    OPC/LCPC methods 
    \cite{cao2020optimized}
    & 
    $\checkmark$
    & 
    $\checkmark$
    & 
    $\checkmark$
    & 
    $\times$  
    &
    Rayleigh
    \\
    Air-FEEL
    \cite{cao2021optimized}
    & 
    $\times$  
    & 
    $\checkmark$
    & 
    $\checkmark$
    & 
    $\times$  
    &
    Block
    \\
    \bottomrule
    \end{tabular}
    }
    \caption{\small The knowledge of channel distribution information from key references.
    }
    \label{tab:cdi knowledge}
    \vspace*{-2.5\baselineskip}
\end{table}

The core challenge in OTA federated learning is to balance strict local energy constraints with the need to achieve an unbiased global update across noisy, heterogeneous MACs.
To ensure the parameter server receives a high-quality aggregated update, each device must meticulously calibrate its transmission message while strictly adhering to local energy constraints.
Earlier work \cite{amiri2020federated,sery2021over,mao2022charles,zhu2024over} lets devices contribute uniformly at random, thereby implicitly assuming identical path losses.
Yang et al. \cite{yang2020federated} ensure unbiased aggregation via channel inversion.
A strict per-device energy constraint forces transmit scaling to be dictated by the weakest device, thereby substantially amplifying noise in the aggregated signal.
Yao et al. in \cite{yao2024wireless} enforce unbiased aggregation 
by mandating truncated channel inversion and dropping out devices with deep fading.
Unfortunately, the path losses are naturally heterogeneous and may even be time-varying in expectation \cite{3gpp38901,iturm1225}, \eg, when devices are mobile.
In scenarios with heterogeneous path losses, 
Abrar and Michelusi in \cite{abrar2023analog} perform channel inversion at every device to achieve unbiased aggregation, but offload devices with weak channels to digital networks.
They focus on minimizing the error of gradient estimation from both digital and analog networks under a delay constraint, 
without convergence guarantees.
In contrast, we establish explicit convergence guarantees for~\fedopi.
Truncated channel inversion is a commonly adopted method to filter out devices with extreme channel conditions.
For example,
Zhu et al. in \cite{zhu2019broadband} schedule only a fixed subset of devices in each round according to their distances to the parameter server; however, their methods remain largely heuristic and inherently introduce bias due to incomplete device population.
Cao et al. in \cite{cao2020optimized} optimize joint transmit-power and denoising-factor to balance the aggregation-coefficient misalignment against receiver noise amplification, leaving learning performance uncharacterized and therefore still leading to a biased convergence.
Cao et al. in \cite{cao2021optimized} optimize the power control of devices to minimize a learning convergence bound, which admits two regimes.
The former admits a biased aggregation that will hurt performance; the latter recovers exact convergence, but each device's power budgets have to be relatively high.
Abrar and Michelusi in \cite{abrar2025non,11475389} design a successive convex approximation framework to simultaneously optimize the learning performance and power control of a biased learning objective under the Rayleigh fading model, while our~\fedopi~can correct bias and account for arbitrary fading models.
\prettyref{tab:cdi knowledge} illustrates the CDI knowledge required by the key references, both at the parameter server and devices. 
It is easy to check that our~\fedopi~requires the least amount of channel distribution information and allows the fading model to be arbitrary.

In federated learning without OTA transmission, dynamic device availability is typically modeled as a discrete network event, independent from wireless transmission errors.
The device updates are transmitted and decoded individually without errors.
The OTA federated learning, however, relies on the superposition property of analog waveforms; the absence of individual device updates fundamentally alters the aggregation protocol, rendering direct applications of their tools infeasible.
We defer the detailed discussions to Appendix A.

\section{Problem Formulation}
\label{sec: problem formulation}
\noindent{\bf Federated Learning System.}
A federated learning system deployed in a wireless network consists of one parameter server and $m$ devices that collaboratively minimize
\begin{align}
    \label{eq:global objective}
    \min_{\x \in \reals^d}
    F (\x) \triangleq
    \frac{1}{m} \sum_{i=1}^m F_i (\x),
\end{align}
where, at device $i$, $F_i (\x) = \Expect_{\xi_i \in \calD_i}[\ell_i (\x;\xi_i)]$ is the local objective, $\calD_i$ is the local distribution, $\xi_i$ is a stochastic sample that device $i$ has access to, and $\ell_i$ is the local loss function, which can be non-convex.

We are interested in solving~\eqref{eq:global objective} over dynamic wireless networks.
Specifically, in each round $t$, the communication uplink between the parameter server and the device $i$ is active with probability $p_i^t$, 
which can be simultaneously time-varying across rounds and heterogeneous across devices.
The activation probability $p_i^t$ will be formally derived next as a function of the fading distribution and a truncation threshold.

\begin{assumption}[Uncertainties in $p_i^t$'s]
\label{ass: threat model}
There exists a $c \in (0,1]$ such that $p_i^t \ge c$, where the activation events of uplink $i$ are independent across devices $i \in [m]$ and across rounds $t \in [T]$.    
\end{assumption}
\prettyref{ass: threat model} is a mild regularity condition on the environment.
For our algorithm to work, neither the parameter server nor the devices are required to know $c$.
Dynamic $p_i^t$'s cause the standard FedAvg to optimize an auxiliary objective.

\begin{remark}[Objective Inconsistency]
    \label{rmk: objective inconsistency}
    In a standard digital FedAvg implementation, each device $i$ computes a local update based on the most recent global model parameters and transmits it to the parameter server.
    When $p_i$'s are heterogeneous across devices but time-invariant across rounds, the learning process minimizes
    a biased objective \cite{wang2023lightweight,11475389}
    \begin{align}
    \label{eq:biased convergence aux}
    \tilde F(\x) = \frac{\sum_{i=1}^{m} p_i F_i (\x)}{\sum_{i=1}^m p_i},    
    \end{align}
    rather than the unbiased objective in \eqref{eq:global objective}.
    Depending on the heterogeneity 
    in $p_i$'s, the output of~FedAvg~can be significantly away from the true optimum of \eqref{eq:global objective}. 
    It is easy to see that time-varying $p_i^t$'s only exacerbate the bias.
\end{remark}
In our setting, $p_i^t$ is dictated by the wireless channel model and the over-the-air computation mechanism.
Without an algorithmic correction, this objective inconsistency is unavoidable.
We next make $p_i^t$'s explicit under channel models.

\noindent{\bf Communication Channel Model.}
Following the common practice of over-the-air federated learning literature~\cite{yang2020federated,xu2021learning,sery2021over}, we assume that the downlink channels, used by the parameter server to broadcast the global model to all devices in a wireless network, are reliable and error-free.

The uplink channel from device $i$ to the parameter server in round $t$ is subject to complex fading $h_{i}^{t} \in \mathbb{C}$ independent across devices and rounds. 
We define the statistical channel gain $\Lambda_{i}^{t} = \Expect [|h_{i}^{t}|^2] > 0$, which may change across rounds due to, \eg, device mobility.
Rayleigh fading, where $h_i^{t} \sim \calC\calN (0, \Lambda_{i}^t)$, is a canonical example. 
We assume $\Lambda_{i}^{t}$ is constant within each round but may change across rounds, and that neither the parameter server nor the devices have access to the statistical CDI $\Lambda_{i}^{t}$'s during training.

\noindent{\bf Over-the-Air Computation.}
We assume each device $i$ knows its own small-scale CDI $h_i^t$, which can be acquired via standard pilot-based channel estimation.
The statistical channel gain $\Lambda_{i}^{t}$ remains unknown.
Based on this, device $i$ uses a truncated channel inversion strategy: if the channel coefficient magnitude exceeds a threshold $\phi_i > 0$, the device transmits a pre-scaled signal; otherwise, no message will be transmitted.
Let $d$ define the model dimension, and $E_s$ the per-dimension energy budget.
Often, the transmitted message $\hat{\bu}_i^{t}$ is a multi-step local gradient sum. 
Formally, $\hat{\bu}_{i}^{t}$,
subject to an energy constraint $\norm{\hat{\bu}_i^t}^2 \le d E_s$ \cite{abrar2023analog,abrar2025non,11475389}, 
can be defined as 
\begin{align}
    \label{eq:transmit signal}
    \hat{\bu}_{i}^{t}
    \triangleq
    \begin{cases}
        \frac{\gamma \bu_{i}^{t}}{h_{i}^{t}},~&\text{if}~ \abth{h_{i}^{t}} \ge \phi_{i},\\
        \bm{0},~&\text{otherwise}.
    \end{cases}
\end{align}
Specifically, it requires $\|\bu_{i}^{t}\|^2\leq d E_s |h_i^t|^2 / \gamma^2$ to satisfy the energy constraint when $|h_i^t| \ge \phi_i$.
When the transmitted message $\bu_i^t$ is the gradient, prior work commonly addresses this via assuming a uniformly bounded gradient \cite{sery2021over,tegin2023federated,11475389,zhu2024over,abrar2025non}.
However, in practice, the value of such a bound may {\em not} be readily available prior to algorithm implementation. 
This motivates our algorithm design in~\prettyref{sec: algorithm}, which does {\em not} require its value as an input.

For the uplink of device $i$ in round $t$, we define its binary activation indicator function as $\delta_{i}^{t} \triangleq \Indc \{\abth{h_{i}^{t}} \ge \phi_{i}\}$.
The transmitted message $\hat{\bu}_{i}^{t}$ can therefore be simplified as $\hat{\bu}_{i}^{t} = \gamma \bu_{i}^{t} \delta_{i}^{t}/{h_{i}^{t}}$.
When the activation events are independent of the transmission message, 
the uplink activation probability is 
\begin{align}
    \label{eq:uplink avail prob}
    p_{i}^{t}
    &\triangleq
    \expect{\delta_{i}^{t}}
    =
    \prob{\abth{h_{i}^{t}} \ge \phi_{i}}
    \in (0,1].
\end{align}
The specific form of $p_i^t$ depends on the fading model, for example, under Rayleigh fading, $p_i^t = \exp (- {\phi_{i}^{2}}/{\Lambda_{i}^t})$.
The parameter server receives the superimposed signal over a noisy MAC and estimates the global aggregation as
\begin{align}
    \label{eq:noisy mac}
    \bm{y}^t
    = \frac{1}{\gamma} \pth{\sum_{i=1}^{m} 
    h_{i}^{t} \hat{\bu}_i^{t} + \z^t}
    = \sum_{i=1}^{m} 
    \delta_{i}^{t} \bu_i^{t} + \frac{\z^t}{\gamma},
\end{align}
where $z^t \sim \calN (0, \sigma_z^2\, \identity)$ is receiver noise.
The pre-scalar $\gamma$ controls a tradeoff between transmit power and effective receiver noise: 
a large $\gamma$ reduces the effective noise term $\z^t / \gamma$ at the cost of requiring higher transmission power, while a smaller $\gamma$ relaxes the power requirement at the expense of increased effective receiver noise.

When $\bu_i^t$ encodes device $i$'s local gradients, taking the expectation with respect to the randomness in $\delta_{i}^{t}$'s and the receiver noise $\bm{z}^t$, we have a weighted average of local gradients $\expects{\bm{y}^t}{\delta,\z} = \sum_{i=1}^m p_i^t \bu_{i}^{t}$, which leads to the biased objective in~\eqref{eq:biased convergence aux} in~\prettyref{rmk: objective inconsistency}.

\section{Algorithm:~\fedopi}
\label{sec: algorithm}
\subsection{Overview}
\label{sec: algorithm overview}
In this section, we propose {\bf Fed}erated {\bf O}ver-the-{\bf A}ir Implicit {\bf G}ossiping
(\fedopi), formally presented in~\prettyref{alg:fedopi online estimate}.
\fedopi~aims to tackle two primary challenges in over-the-air computation: 
(i) energy-constrained transmission over noisy MACs and
(ii) unbiased convergence without enforcing all device transmission.
To address these challenges,~\fedopi~introduces two key mechanisms:
(i) truncated channel inversion with gradient normalization that guarantees energy compliance and partial device contributions,
and (ii) postponed global model broadcast that enables a balanced information mixture through implicit device-device gossip.
Next, we provide algorithm design details.

In each round $t$, device $i$ performs $s$-steps of local stochastic gradient descent on its own model $\x_{i}^{t}$, which generates a local iterate $\tilde \x_i^{t}$ (lines 4-9). 
Let $\eta$ denote a learning rate, 
and $\tau_i(t)$ denote the most recent round prior to $t$ in which device $i$ successfully uploaded to the parameter server, with the convention $\tau_i(0) = -1$ for $i \in [m]$.
The multi-step local update is then defined as $\Delta_{i}^{t} = (\tilde \x_{i}^{t} - \x^{\tau_i(t)+1}) / \eta$, which characterizes the local cumulative unsent gradient innovation since the device's $\tau_i(t) + 1$.

Notably, each device continues to compute fresh gradients in each round regardless of its respective uplink conditions. 
Departing from the standard FedAvg algorithm, where each device updates using the most recent global model,~\fedopi~has each device update its own local model, which may be stale depending on its uplink conditions.
In the special case where $c=1$, the staleness vanishes, and~\fedopi~reduces to standard FedAvg.
It turns out that such updates help mitigate the bias caused by heterogeneous and time-varying uplinks. 
We show that the staleness is bounded in expectation 
by $1/c$ in~Lemma 7 in~Appendix D,
and there is no significant empirical slowdown observed in~\prettyref{fig:final results} in~\prettyref{sec: numerical experiment}.

\noindent{\bf Truncated channel inversion with gradient normalization (Energy Compliance).}
In fading MACs, weak channels might cause arbitrarily large amplification of the transmitted message when inverted. 
We adopt the truncated channel inversion~\eqref{eq:transmit signal} at device $i$ in round $t$ (lines 11 - 17). 
Specifically, device $i$ transmits only when $|h_{i}^{t}| \ge {\gamma}/{\sqrt{d E_s}}$. 
Consequently, in highly dynamic wireless channels, not all devices will transmit.
Let $\calA^t \triangleq \{i: |h_{i}^{t}| \ge {\gamma}/{\sqrt{d E_s}}\}$ denote the collection of devices at round $t$. 
For ease of exposition, we refer to devices in $\calA^t$ as devices with active uplinks in round $t$. 
The transmission message $\hat{\x}_{i}^{t}$ at device $i$ in round $t$ is designed as
\begin{align}
\label{eq: message design}
    \hat{\x}_{i}^{t}
    &=
    \begin{cases}
        \frac{\gamma}{h_{i}^{t}}
        \frac{\Delta_{i}^{t}}{B_t},
        &~\text{if}~\abth{h_{i}^{t}} \ge \frac{\gamma}{\sqrt{d E_s}}, \\
        \bm{0},&~\text{otherwise},
    \end{cases}
\end{align}
where $\gamma$ is a pre-scalar to determine the channel quality relative to the energy constraint $\sqrt{d E_s}$, and $B_t \triangleq \max_{i \in \calA^t} \norm{\Delta_i^t}$ is an online norm tracker to obtain the maximum norm of local updates across the devices with active uplinks. 

In practice, $\gamma$ can be tuned to obtain the best empirical performance.
The scalar $B_t$ can be determined through a bilateral lightweight coordination between the parameter server and the devices. Upon observing $h_{i}^{t}$, devices with active uplinks report their one-dimensional $\norm{\Delta_{i}^{t}}$ to the parameter server over high-quality digital uplinks such as FDMA, OFDMA, and TDMA \cite{3gpp38901,iturm1225}. 
The parameter server then computes $B_t = \max_{i \in \calA^t} \norm{\Delta_{i}^{t}}$ and sends it back to the devices in $\calA^t$. The additional communication overhead is $\calO (1)$ per device and therefore negligible relative to the $d$ dimensional model parameter transmission.

On the technical front, our design guarantees 
(i) energy constraint compliance for each device $i$ in $\calA^t$, and 
(ii) independent uplink activation events without requiring a known gradient bound.
For (i), observe that for any device $i$ in $\calA^t$ in round $t$, we have
\begin{align*}
    |h_{i}^{t}| \ge \frac{\gamma}{\sqrt{d E_s}}
    & \Leftrightarrow
    \frac{\gamma}{\abth{h_{i}^{t}}} \le \sqrt{d E_s}
    \Rightarrow
    \norm{\hat{\x}_{i}^{t}}^2 =
    \norm{\frac{\gamma}{h_{i}^{t}} \frac{\Delta_i^t}{B_i^t}}^2 
    \le
    \norm{\frac{\gamma}{h_{i}^{t}}}^2 
    \le d E_s.
\end{align*}
For (ii), since energy constraint is guaranteed by construction, the uplink activation event $\{i \in \calA^t\}$ depends only on the fading channel threshold $\abth{h_{i}^{t}} \ge \gamma / \sqrt{d E_s}$, and thus inherits the independence structures of the fading variable $h_{i}^{t}$ across devices and rounds.
Our design is also more general than most prior OTA works, which often require a uniformly bounded gradient norm as an input \cite{sery2021over,tegin2023federated,11475389,zhu2024over,abrar2025non}, a condition that is hard to obtain in advance of any algorithm implementations.

Upon receiving $\bm{y}^{t}$ as in~\eqref{eq:noisy mac},
the parameter server rescales the received noisy message by $B_t / (\gamma |\calA^t|)$,
thereby amplifying the receiver noise accordingly: the effective noise variance in the recovered update is $B_t^2 \sigma_z^2 / (\gamma^2 |\calA^t|^2)$ (line 18). 
$\gamma$ governs a fundamental tradeoff between uplink availability and signal quality.
Intuitively, $\gamma \approx 0$ allows all devices to satisfy the energy constraint and transmit their local updates, but the resulting noise amplification renders the gradient information unusable. Conversely, a large $\gamma$ suppresses noise but excludes devices with weak channels from contributions to the global training. 
We will characterize its impacts theoretically in~\prettyref{rmk: convergence rate} in~\prettyref{sec: convergence analysis} and empirically in~\prettyref{fig:impacts of gamma} in~\prettyref{sec: numerical experiment}.

\noindent{\bf Postponed Global Model Broadcast (Debias Updates).}
Unlike standard FedAvg, where the parameter server broadcasts the global model at the {\em beginning} of each round,~\fedopi~employs a {\em postponed broadcast}. 
In~\fedopi, the parameter server collects updates from devices with active uplinks, 
and then sends the updated global model $\x^{t+1}$ back to devices in $\calA^t$ at the {\em end} of each round (lines 22-23).

Note that the aggregated signal $\breve{\Delta}^t$ is computed relative to different reference models $\{\x^{\tau_{i}(t) + 1}\}_{i \in \calA^t}$, the parameter server must compensate for the reference models accordingly to recover the original updated model parameter $\tilde \x_{i}^{t}$ (line 19).
This is because the local cumulative unsent innovation is calculated relative to the updated global, which is broadcast back to the device at the {\em end} of its last successful transmission round $\tau_i(t)$.
Formally, we have
\begin{align*}
\x^{t+1}
= \eta \breve{\Delta}^t + 
\frac{1}{\abth{\calA^t}} \sum_{i \in \calA^t} \x^{\tau_i(t) + 1}     
=
\frac{1}{\abth{\calA^t}} \sum_{i \in \calA^t} \tilde \x_{i}^{t}
+
\frac{\eta B_t}{\gamma \abth{\calA^t}}
\bm{z}^{t},
\end{align*}
where $\bm{z}^t \sim \calN(\bm{0},\, \sigma_z^2\ \identity)$ is the receiver noise.
Notably, these reference models are originally generated by the parameter server; the server simply retrieves them from a checkpoint buffer $\calC$. 
Thus, we will not incur any additional communication overhead to obtain such computation anchors. 
In line 21, we remove outdated models from $\calC$ that will no longer be referenced by any devices. 
In the worst-case scenario where all devices hold different $\tau_i (t)$'s, the parameter server will incur an additional memory overhead of $\calO (m d)$.

\subsection{Implicit Gossiping}
\label{subsec: implicit gossiping}
Through postponed broadcast in Algorithm \ref{alg:fedopi online estimate}, devices in $\calA^t$ {\em implicitly} gossip their local updated models with each other through the parameter server \cite{xiang2023towards,xiang2024efficient,xiang2025empowering}.
Unlike standard gossiping protocols, no direct device-to-device communication is involved.

Gossip-type algorithms were originally used for device-to-device networks and are well-known for their resilience to reach average consensus despite communication failures and asynchronous information exchange
\cite{boyd2006randomized,Lynch:1996:DA:2821576,nedic2009distributed}. 
At a high level, the devices' local estimates are ultimately weighted evenly in the final model output. 
Formally, the information mixing matrix $W$ is constructed as a doubly stochastic matrix in~\eqref{eq: gossiping matrix}.
\begin{align}
\label{eq: gossiping matrix}
W_{ij}^{(t)} =
\begin{cases} 
\frac{1}{\abth{\calA^t}}, &~~~~ \text{if }i, j\in \calA^t; \\    
1, &~~~~ \text{if } i=j\,\text{and} \, i\notin \calA^t;\\ %
0, &~~~~ \text{otherwise}. 
\end{cases}
\end{align}
The information mixing matrix $W$ can be time-varying even in expectation due to the high uncertainties in $p_i^t$'s. Let $M^{(t)} \triangleq \Expect [(W^{(t)})^2]$, $\rho(t) \triangleq \lambda_2 (M^{(t)})$, $\allones \triangleq \Indc \Indc^{\top} / m$, and $\rho \triangleq \max_t \rho(t)$, where $\lambda_2(\cdot)$ denotes the second largest eigenvalue.
\prettyref{lmm: spectral norm main} characterizes the information mixing error, \ie, consensus error.
\begin{lemma}[\cite{nedic2017achieving,nedic2018network}]
\label{lmm: spectral norm main}
For any matrix $B \in \reals^{d \times m}$ that is independent of $W$ matrices, it holds that
\begin{align*}
\expect{\Bigg\|B \pth{\prod_{r=1}^{t} W^{(r)} - \allones}\Bigg\|_{\mathrm{F}}^2 ~\Big|~ B} \le \rho^t \fnorm{B}^2, 
\end{align*}
where $\fnorm{\cdot}$ denotes the Frobenius norm of a matrix. 
\end{lemma}

\begin{lemma}[Ergodicity \cite{xiang2023towards,xiang2024efficient,xiang2025empowering}]
\label{lmm: ergodicity}
Suppose that~\prettyref{ass: threat model} holds.
For all $t\ge 1$, it holds that 
$\rho = \max_{t} \rho(t) \le 1 - \frac{c^4\qth{1-\pth{1-c}^m}^2}{8}$.
\end{lemma}

\begin{algorithm}[!t]
    \textbf{Input:} 
    Total number of rounds: $T$, 
    model initialization: $\x^0$, 
    number of local steps: $s$, 
    learning rate: $\eta$,
    pre-scalar $\gamma$, 
    and checkpoint buffer $\calC = \emptyset$\;
    \textbf{Initialization:} 
    $\x_{i}^{0} \gets \x^{0}$,
    $\tau_i(0) \gets -1$, $\forall i \in [m]$,
    $\calC = \{(0, \x^{0})\}$
    \;

    \For{$t=0, \cdots, T-1$}
    {
        \tcc{\color{blue} On all devices.}
        \For{$i\in [m]$}
        {
            Draw a fresh sample $\xi_{i}^{t}$\;
            $\x_i^{(t,0)} \gets \x_i^{t}$\;
            \For{$k=0, \cdots, s-1$}
            {
                $\x_i^{(t, k+1)} \gets \x_i^{(t, k)} - \eta \nabla \ell_i(\x_i^{(t, k)}; \xi_{i}^t)$\;
            }
            $\tilde{\x}_i^{t} \gets \x_i^{(t,s)}$\;
            Compute local cumulative unsent innovation $\bm{\Delta}_i^t \gets (\tilde{\x}_i^t - \x^{\tau_i(t) + 1}) / \eta$
            \;
        }

        \tcc{\color{olive} Decide active uplinks.}
        $\calA^t \gets \sth{i \in [m] \mid \abth{h_{i}^{t}} \ge \frac{\gamma}{\sqrt{d E_s} }}$\;

        \If{$\calA^t \neq \emptyset$}{
            \tcc{\color{blue} On the devices with active uplinks.}
            \For{$i \in \calA^t$}
            {
                Report $\norm{\bm{\Delta}_i^t}$ to the parameter server\;
            }
            \tcc{\color{blue} On the parameter server.}
            $B_t \gets \max_{i \in \calA^t} \norm{\bm{\Delta}_i^t}$\;

            \For{$i \in \calA^t$}
            {
                $\hat{\x}_i^{t} \gets \dfrac{\gamma}{h_{i}^{t}} \dfrac{\bm{\Delta}_i^t}{B_t}$\;
            }

            \tcc{\color{blue} On the parameter server.} 
            $\breve{\bm{\Delta}}^{t} \gets \dfrac{B_t}{\gamma \abth{\calA^t}}
                \pth{\bm{z}^t + 
                \sum_{i \in \calA^t} h_i^t \hat{\x}_i^t}$,
            where $\bm{z}^t \sim \calN(\bm{0},\, \sigma_z^2\ \identity)$\;

            \tcc{\color{olive} Model reconstruction.}
            ${\x}^{t+1} \gets \eta \breve{\bm{\Delta}}^t +
                \dfrac{1}{|\calA^t|} \sum_{i \in \calA^t} \x^{\tau_i(t) + 1}$\;

            \tcc{\color{olive} Local checkpoint buffer update.}
            $\calC \gets \calC \cup \{(t+1,\, \x^{t+1})\}$
            \tcp*{store latest model}
            $\calC \gets \{(r,\, \x^r) \in \calC \mid r \geq \min_{j \in [m]} \tau_j(t+1) + 1\}$
            \tcp*{prune outdated models}
            
            \tcc{\color{olive} Postponed broadcast.}
            \For{$i \in \calA^t$}
            {
                $\x_i^{t+1} \gets \x^{t+1}$,
               $\tau_i(t+1) \gets t$\;
            }
            \lElse{
            $\x_i^{t+1} \gets \tilde{\x}_i^t$,
            $\tau_i(t+1) \gets \tau_i(t)$
            }
        }
        \Else{$\x^{t+1} \gets \x^t$\;
        \lFor{$i\in [m]$}{$\x_i^{t+1} \gets \tilde{\x}_i^t$,
            $\tau_i(t+1) \gets \tau_i(t)$}} 
        }
    \caption{\fedopi}
    \label{alg:fedopi online estimate}
\end{algorithm}

\section{Convergence Analysis}
\label{sec: convergence analysis}
In this section, we analyze the convergence of~\fedopi.
All missing proofs and intermediate results are deferred to Appendix.

\subsection{Assumptions and Preliminaries}
\label{sec: assumptions}
We first introduce regularity assumptions and preliminary results that are used in our convergence analysis. 
Let $\calF^t$ denote the natural filtration generated by the randomness up to round $t$.
\begin{assumption}[Smoothness]
\label{ass: 2 smmothness}
For each device node $i\in [m]$, its local stochastic gradient function
$\nabla \ell_{i}(\x, \xi)$ is $L_i$-Lipschitz, \ie, for any stochastic sample $\xi$,
$$
\norm{\nabla \ell_{i}(\x_1; \xi)-\nabla \ell_{i}(\x_2; \xi)}
\le L_i \norm{\x_1-\x_2}
\le L \norm{\x_1-\x_2}
, ~~~~ \forall \x_1, \x_2,
$$ 
where $L \triangleq \max\limits_{i\in[m]} L_i$. 
\end{assumption}
\begin{assumption}[Bounded Variance]
\label{ass: bounded variance client-wise}
Stochastic gradients at each device node $i\in[m]$ are unbiased estimates of the true gradient of the local objectives, i.e., 
\[
\expect{\nabla \ell_i(\x_i^t; \xi_{i}^{t}) \mid \calF^{t}}=\nabla F_i(\x_i^t),  
\]
and the variance of stochastic gradients at each device node $i\in[m]$ is uniformly bounded, i.e., 
\[
\expect{\norm{\nabla \ell_i(\x_{i}^{t}; \xi_{i}^{t})-\nabla F_i(\x_{i}^{t})}^2 \mid \calF^t}\le\sigma^2, ~~~ \forall ~ \x \in \reals^d. 
\]
\end{assumption}
Notably, in Algorithm \ref{alg:fedopi online estimate}, although each round consists of $s$ local steps, only a single sample $\xi_i^t$ is drawn. The unbiasedness in \prettyref{ass: bounded variance client-wise} is imposed with respect to $\x_i^t$, {\bf rather than} $\x_i^{(t,k)}$ for all $k=0, ... , s-1$. 
In other words, Assumption \ref{ass: bounded variance client-wise} is more relaxed than the assumption of unbiased stochastic gradients at each local step. 
To account for this relaxation, following \cite{su2023federated,xiang2023towards,xiang2025empowering}, we bound the cumulative deviation of $\nabla \ell_i(\x_i^{(t,k)}; \xi_i^t)$ from $\nabla \ell_i(\x_i^{t}; \xi_i^t)$; details can be seen in Lemma \ref{lmm: local step perturbation}. 
\begin{lemma}[\cite{su2023federated,xiang2023towards,xiang2025empowering}]
\label{lmm: local step perturbation}
Suppose~\prettyref{ass: 2 smmothness} holds. When $s> 1$,   
\[
\|\sum_{k=0}^{s-1} 
[\nabla \ell_i (\x_i^{\pth{t,k}};\xi_i^{t}) - \nabla \ell_i (\x_i^t; \xi_i^{t})]\|_2
\le \frac{\kappa \eta s(s-1) L_i}{2} \norm{\nabla \ell_{i}(\x_i^t;\xi_i^{t})},
\]
where
$
\kappa \triangleq \max_{i}\frac{2 \pth{(1+\eta L_i)^s - 1- s \eta L_i}}{
s(s-1)
\pth{\eta L_i}^2} 
$
and non-decreases with respect to $\eta>0$.
When $s=1$, $\|\nabla \ell_i (\x_i^{\pth{t,0}};\xi_i^{t}) - \nabla \ell_i (\x_i^t; \xi_i^{t})\|_2 =0$.  
\end{lemma}
It is easy to check that $\kappa$ remains bounded from above as long as the learning rate $\eta$ is bounded. 
For example, when $\eta \le \epsilon/ (s L)$, we have $\kappa \le 2 (e^\epsilon - 1 -\epsilon) / \epsilon^2$.
That is, we can treat $\kappa$ as a {\em constant} when $\eta$ is sufficiently small.

\begin{assumption}[Bounded Inter-device Heterogeneity]
\label{ass: bounded similarity}
We say that local objective function $F_i$'s satisfy $(\beta,\zeta)$-bounded dissimilarity condition for $\beta,\zeta \ge 0$ if
\begin{align}
\label{eq: BG condition}
\frac{1}{m}\sum_{i=1}^m \norm{\nabla F_i(\x)- \nabla F(\x)}^2 \le \beta^2 \norm{\nabla F(\x)}^2+ \zeta^2, ~~~ \forall ~ \x \in \reals^d.
\end{align} 
\end{assumption}

Assumptions~\ref{ass: 2 smmothness}, and~\ref{ass: bounded variance client-wise} are standard in federated learning analysis \cite{karimireddy2020scaffold,li2020federated,yuan2022}.
Assumption \ref{ass: bounded similarity} captures the heterogeneity across different users.
When devices have identical local datasets, 
it holds for~\eqref{eq: BG condition} that $\beta = \zeta = 0$ since $F_i = F_j$.

\begin{assumption}[Bounded True Gradient]
\label{ass: bounded true gradient}
The full-batch gradient is uniformly bounded, \ie, 
$\norm{\nabla F_i (\x)} \le G_{\max}$, $\forall \x \in \reals^d$.
\end{assumption}
\prettyref{ass: bounded true gradient} assumes that the true gradients are bounded from above. 
It is more relaxed than those commonly adopted in the OTA literature \cite{sery2021over,tegin2023federated,11475389,zhu2024over,abrar2025non}, where stochastic gradients are typically required to be uniformly bounded. We conjecture that the gradient boundedness assumption is primarily imposed for technical convenience, and we leave its relaxation for future work. 
Notably, it is easy to check that implementing our algorithm does not require $G_{\max}$ as an input.

\subsection{Convergence Results}
\label{sec: convergence results}
Following the widely adopted roadmap of gossip-type learning algorithm analysis
\cite{nedic2009distributed,nedic2017achieving,nedic2018network,wang2022matcha},
we focus on the convergence of $\bar{\x}^t \triangleq \frac{1}{m}\sum_{i=1}^m \x_i^t$,  
while our numerical experiments validate the $\x^t$. 
\begin{lemma}[Descent Lemma]
\label{lmm: descent lemma}
Suppose Assumptions \ref{ass: 2 smmothness}, \ref{ass: bounded variance client-wise}, and \ref{ass: bounded similarity} hold.
Choose a learning rate $\eta$ such that
$\eta \le \frac{1}{108s L (\kappa + 1) (\beta^2+1)}$.
It holds that
\begin{align*}
&\expect{F(\bar{\x}^{t+1})  -  F(\bar{\x}^{t}) \mid \calF^{t}}
\le -\frac{\eta s }{3} \norm{\nabla F(\bar{\x}^t)}^2 
+
\underbrace{\frac{\eta^2  L d \sigma_z^2}{\gamma^2 m^2}\expect{B_t^2 \mid \calF^t}}_{\text{Receiver Noise}}
\\
&\qquad  \qquad
+ 6 \eta^2 s^2  L\pth{\zeta^2+\sigma^2}\pth{1+\kappa^2 L^2}
+  \frac{\eta s L^2 }{m} \sum_{i=1}^m 
\underbrace{\norm{\x_i^t - \bar{\x}^t}^2}_{\text{Consensus Error}}. 
\end{align*} 
\end{lemma}
\begin{proof}[Proof Sketch]
Unroll one iteration of $\bar{\x}^{t+1}$, it holds that    
\begin{small}
\begin{align}
    \notag
    \bar{\x}^{t+1} - \bar{\x}^{t}
    &=
    \frac{1}{m} 
    \sum_{i \in \calA^t}
    \pth{
    \frac{1}{\abth{\calA^t}} 
    \pth{\frac{\eta B_t \bm{z}^t}{\gamma}
    +
    \sum_{i\in \calA^t} \tilde \x_i^{t} 
    }
    - \x_i^t} 
    +
    \frac{1}{m} 
    \sum_{i \not \in \calA^t}
    \pth{
    \tilde \x_i^{t}
    - \x_i^t
    } \\
    \label{eq: unroll one round}
    &=
    -
    \frac{\eta}{m} 
    \sum_{i=1}^m
    \sum_{r=0}^{s-1}
    \nabla \ell_{i}(\x_{i}^{(t,r)})
    +
    \frac{\eta }{m} \frac{B_t \bm{z}^t}{\gamma}
    .
\end{align}
\end{small}
Via~\prettyref{ass: 2 smmothness}, we have
\begin{small}
\begin{align*}
&F(\bar{\x}^{t+1})  -  F(\bar{\x}^{t}) 
\le 
\underbrace{\iprod{\nabla F(\bar{\x}^{t})}{- \frac{\eta }{m} 
\sum_{i=1}^m 
\sum_{r=0}^{s-1}
\nabla \ell_i(\x_i^{(t,r)};\xi_{i}^{t})
}
}_{(A)} \\
&~~~+  
\underbrace{\iprod{\nabla F(\bar{\x}^{t})}{\frac{\eta}{m} \frac{B_t}{\gamma} \bm{z}^{(t)}} }_{(B)}
+
\frac{L}{2}
\underbrace{\norm{\frac{\eta}{m} 
\sum_{i=1}^m 
\sum_{r=0}^{s-1}
\nabla \ell_i(\x_i^{(t,r)};\xi_{i}^{t})
- 
\frac{\eta B_t \bm{z}^t}{\gamma m}}^2
}_{(C)}
.  
\end{align*}
\end{small}
Term $(A)$ can be bounded by following the standard roadmap of stochastic gradient analysis with~\prettyref{lmm: local step perturbation} and Assumptions~\ref{ass: bounded variance client-wise} and~\ref{ass: bounded similarity}.
Recall that $B_t \triangleq \max_{i \in \calA^t} \|\Delta_{i}^{t}\|_2^2$ is an online norm tracker to obtain the maximum norm of local updates in $\calA^t$ in round $t$.
The key observation is that, conditional on $\calF^t$, $B_t$ and receiver noise $\bm{z}^t$ are independent. So, we have 
\begin{small}
\begin{align*}
    \expect{(B) \mid \calF^t}
    &=
    \frac{\eta}{m \gamma}
    \iprod{\nabla F(\bar{\x}^{t})}{
    \expect{B_t \mid \calF^t}
    \expect{\bm{z}^{t} \mid \calF^t} }= 0
\end{align*}
\end{small}
Next, we use Jensen's inequality to upper bound the term $(C)$:
\begin{small}
\begin{align*}
    \expect{(C) \mid \calF^t}
    &\le
    \frac{2 \eta^2}{m^2}
    \underbrace{\norm{
    \sum_{i=1}^m 
    \sum_{r=0}^{s-1}
    \nabla \ell_i(\x_i^{(t,r)};\xi_{i}^{t})
    }^2}_{(C.{\rm I})}
    +
    \frac{2 \eta^2}{\gamma^2 m^2 }
    \underbrace{\norm{B_t \bm{z}^{t}}^2}_{(C.{\rm II})}
    .
\end{align*}
\end{small}
By the conditional independence of $B_t$ and $\bm{z}^t$ on $\calF^t$, we have $(C.{\rm II}) = \Expect [\|B_t \bm{z}^t\|_2^2 | \calF^t] \le d \sigma_z^2 \Expect[B_t^2 | \calF^t]$.
Term $(C.{\rm I})$ can be upper bounded by invoking Assumptions~\ref{ass: bounded variance client-wise},~\ref{ass: bounded similarity}, and~\prettyref{lmm: local step perturbation}.
The rest of the proof combines the intermediate results of terms $(A)$, $(B)$, and $(C)$, and by applying our learning rate condition $\eta \le {1}/(108s L (\kappa + 1) (\beta^2+1))$. 
Unlike FedAvg algorithm,~\prettyref{lmm: descent lemma} highlights a consensus error term arising from the implicit gossiping and receiver noise term from OTA computation.
\end{proof}

Performing a direct analysis of $B_t$ is challenging due to the high uncertainties of $\calA^t$ and the heterogeneous updates of the device. Recall that device $i$ computes $\Delta_i^t$ relative to $\x^{\tau_i(t) + 1}$, where $\tau_i(t) \triangleq \{t^{\prime} : t^{\prime} < t ~\text{and}~ i \in \calA^{t^{\prime}}\}$ defines the most recent round when device $i$ had an active uplink.
As such, we define an auxiliary variable $C_t^2 \triangleq \sum_{i=1}^{m} \norm{\Delta_i^t}^2$. It is easy to see that $B_t^2 = \max_{i\in\calA^t} \|\Delta_{i}^{t}\|_2^2\le C_t^2$. Note that $C_t^2$ is never computed by devices but is used as a tool to aid our analysis.
\prettyref{lmm:gradient Bt} presents the result.

\begin{lemma}
    \label{lmm:gradient Bt}
    Suppose Assumptions~\ref{ass: threat model},~\ref{ass: 2 smmothness}, %
\ref{ass: bounded variance client-wise}, 
\ref{ass: bounded similarity}, %
and \ref{ass: bounded true gradient} 
hold, and $\eta \le \frac{1}{s L (\kappa + 1)}$.
It holds that
\begin{small}
\begin{align*}
    \frac{1}{T}\sum_{t=0}^{T-1}\expect{B_t^2}
    &\le
    \frac{1}{T}\sum_{t=0}^{T-1}\expect{C_t^2} 
    \le
    \frac{12 m s^2 
    \pth{\sigma^2 + G_{\max}^2} }{c^2}.
\end{align*}
\end{small}
\end{lemma}
\begin{proof}[Proof Sketch]
Expand $C_t^2$, and by Jensen's inequality, we have
\begin{small}
\begin{align*}
    C_t^2 
    =
    \sum_{i=1}^m
    \norm{\Delta_{i}^{t}}^2
    &=
    \sum_{i=1}^m
    \norm{\sum_{k=\tau_i(t) + 1}^{t}
    \sum_{r=0}^{s-1} \nabla \ell_i (\x_i^{(k,r)}; \xi_i^{k})}^2 \\
    &\le
    \sum_{i=1}^m
    (t - \tau_i(t))
    \sum_{k=\tau_i(t) + 1}^{t}
    \norm{\sum_{r=0}^{s-1} \nabla \ell_i (\x_i^{(k,r)};\xi_i^{k})}^2
    .
\end{align*}
\end{small}
By taking the expectation and averaging over $T$ rounds, we have
\begin{small}
\begin{align}
    \frac{1}{T}\sum_{t=0}^{T-1}\expect{C_t^2}
    &\le
    \frac{1}{T}
    \sum_{t=0}^{T-1}
    \sum_{i=1}^m
    \Expect\Bigg[
        (t - \tau_i(t))
        \sum_{k=\tau_i(t) + 1}^{t}
        \norm{\sum_{r=0}^{s-1} \nabla \ell_i (\x_i^{(k,r)};\xi_i^{k})}^2
    \Bigg].
    \label{eq:bt_start main text}
\end{align}
\end{small}
The key step is to regroup the terms in \eqref{eq:bt_start main text}: 
\begin{small}
\begin{align*}
    \eqref{eq:bt_start main text}
    &=
    \frac{1}{T}
    \sum_{i=1}^m
    \sum_{k=0}^{T-1}
    \Expect
    \Bigg[
         \underbrace{\norm{
       \sum_{r=0}^{s-1} \nabla \ell_i (\x_i^{(k,r)};\xi_i^{k})}^2}_{
        \triangleq\; g_{i,k}
        }
        \underbrace{
        \sum_{t=k}^{T-1}
        (t - \tau_i(t))
        \indc{\tau_i(t) + 1 \le k}
        }_{\triangleq\; S_{i,k}}
    \Bigg].
\end{align*}
\end{small}
Crucially, conditional on $\calF^{k}$, 
the multi-step gradient sum $g_{i,k}$ and $S_{i,k}$ are {\em independent} because future channel randomness and the current round stochastic gradients are independent.
Therefore, we have
\begin{small}
\begin{align}
\label{eq: conditional independence}
\expect{g_{i,k} \cdot S_{i,k} ~\Big |~\calF^k}
&=
\expect{g_{i,k} ~\Big |~\calF^k}
\expect{S_{i,k} ~\Big |~\calF^k}.    
\end{align}
\end{small}
We first bound $S_{i, k}$. 
We define an auxiliary look-ahead variable $\tau_{i}^{+}(k) \triangleq \min\{t > k : i \in \calA^t\}$ to denote the
first round after $k$ at which uplink $i$ becomes active.
Since $\indc{\tau_i(t) + 1 \le k} = 0$ for all $t \ge \tau_{i}^{+}(k) + 1$, the
sum $S_{i,k}$ accumulates only until device $i$'s next contribution:
\begin{small}
\begin{align*}
    &S_{i,k}
    =
    \sum_{t=k}^{\tau_{i}^{+}(k) }(t - \tau_i(t))
    \overset{(a)}{=}
    \sum_{j=0}^{\tau_{i}^{+}(k) - k} \pth{j + k - \tau_i(k)} \\
    &=
    \frac{(\tau_{i}^{+}(k) - k)^2}{2}
    +
    \frac{\tau_{i}^{+}(k) - k}{2} 
    +(k - \tau_i(k)) \pth{\tau_{i}^{+}(k) - k}
    +(k - \tau_i(k)),
\end{align*}
\end{small}
where equality $(a)$ follows from change of variable $j = t - k$ and the fact that $\tau_i (k) = \tau_i(j + k)$ for $k \le j + k  \le \tau_i^{+}(k)$ because the uplink is inactive during that period.
We can then invoke
Lemma 7 in~Appendix D
to bound the first moment of the staleness factor $(\tau_i^{+}(k) - k)$ by $1/c$ and its second moment by $2/c^2$.
\prettyref{lmm: local step perturbation} is used to bound $g_{i,k}$.
Combining them together, we have
\begin{small}
\begin{align*}
    &\expect{g_{i, k} S_{i, k} ~\big | ~ \calF^k } 
    =
    \expect{ \frac{{(\tau_{i}^{+}(k) - k)^2}
    +{\tau_{i}^{+}(k) - k} }{2} 
      ~\big | ~ \calF^k}
     \expect{g_{i, k} ~\big | ~\calF^k}\\
    &~~~+
    \expect{(k - \tau_i(k)) 
    \pth{\tau_{i}^{+}(k) - k} 
    +(k - \tau_i(k)) ~\big | ~ \calF^k} 
    \expect{g_{i, k} ~\big | ~\calF^k}\\
    &\le
    \qth{\frac{1}{c^2} + \frac{1}{2 c}
    + \pth{\frac{1}{c} + 1}
    (k - \tau_i(k))
    } s^2 \pth{\kappa^2 \eta^2 s^2 L^2 + 2}
    \pth{\sigma^2 + \norm{\nabla F_i(\x_i^k)}^2}.
\end{align*}
\end{small}
We use~\prettyref{ass: bounded true gradient} to bound $\norm{\nabla F_i (\x_i^k)} \le G_{\max}$.
Observe that the only remaining randomness comes from $(k -\tau_i(k))$.
Taking expectation over the remaining randomness and reusing~Lemma 7 to bound $(k - \tau_i(k))$ yields~\prettyref{lmm:gradient Bt} under $\eta \le 1/ (sL(\kappa + 1)$.
\end{proof}
Next, we bound the consensus error in~\prettyref{lmm: consensus}.
To facilitate analysis, we introduce the following compact matrix forms:
\begin{small}
\begin{align*}
\label{eq: compact matrix}
\bm{X}^{(t)} & = \qth{\x_1^t, \cdots, \x_m^t}; ~
\bm{G}^{(t)} = \Big[{\sum_{r=0}^{s-1}\nabla \ell_1(\x_1^{(t,r)}), \cdots,  \sum_{r=0}^{s-1}\nabla \ell_m(\x_m^{(t,r)})}\Big] ;\\
\bm{Z}^{(t)} & = \frac{B_t}{\gamma}
\frac{\indc{\abth{\calA^t} > 0}}{\abth{\calA^t} + \indc{\calA^{t} = 0}}
\qth{\bm{z}^t \indc{1 \in \calA^t}, \cdots, \bm{z}^t \indc{m \in \calA^t}},
\end{align*}
\end{small}
where $\bm{X}^{(t)}$ is a concatenated model parameter matrix for all devices at the {\em beginning} of round $t$,
$\bm{G}^{(t)}$ is a local stochastic gradient accumulation,
and $\bm{Z}^{(t)}$ is a scaled receiver noise matrix.
When $\calA^t \not = \emptyset$, the receiver noise is rescaled by $B_t / (\gamma \abth{\calA^t})$.
When $\calA^t = \emptyset$, it is easy to check that we simply have $\bm{Z}^{t} = 0 \identity$.

\begin{lemma}[Consensus Error]
\label{lmm: consensus}
Suppose Assumptions \ref{ass: threat model},
\ref{ass: 2 smmothness}, 
\ref{ass: bounded variance client-wise}, 
\ref{ass: bounded similarity},
and \ref{ass: bounded true gradient} hold.
Choose a learning rate $\eta$ such that
$\eta \le
\frac{1 - \sqrt{\rho}}{632 s L (\kappa + 1) \sqrt{\beta^2 + 1}}$.
It holds that
\begin{small}
\begin{align*}
\frac{1}{mT}\sum_{t=0}^{T-1}
\expect{\norm{\x_{i}^{t} - \bar{\x}^{t}}^2} 
&\le
\frac{108 \rho \eta^2 s^2 (\beta^2 + 1)}{(1 - \sqrt{\rho})^2}
\frac{1}{T} \sum_{t=0}^{T-1}
\expect{\norm{\nabla F (\bar{\x}^t)}^2} 
\\
&~~~+ \frac{36 \rho \eta^2 s^2 \sigma^2}{(1-\sqrt{\rho})^2} 
+ \frac{96 \rho m \eta^2 s^2 d \sigma_z^2 (\sigma^2 + G_{\max}^2) }{c^2 {\gamma^2 (1-\sqrt{\rho})^2}}.
\end{align*}
\end{small}
\end{lemma}
\begin{proof}[Proof Sketch]
Recall that $\allones = \frac{1}{m}\Indc \Indc^{\top}$ .
Unroll the recursion, we have
\begin{small}
\begin{align*}
    &\fnorm{\bm{X}^{(t)} (\identity - \allones)}^2
    \overset{(a)}{=} \eta^2
    \fnorm{\sum_{q=0}^{t-1} 
    \pth{\bm{G}^{(q)} - \bm{Z}^{(q)}} 
    \pth{\prod_{l=q}^{t-1} W^{l} - \allones}}^2 \\
    &\qquad
    \overset{(b)}{\le}
    2\fnorm{\sum_{q=0}^{t-1} 
    \bm{G}^{(q)}
    \pth{\prod_{l=q}^{t-1} W^{l} - \allones}}^2 
    +
    2\fnorm{\sum_{q=0}^{t-1} 
    \bm{Z}^{(q)}
    \pth{\prod_{l=q}^{t-1} W^{l} - \allones}}^2 
    ,
\end{align*}
\end{small}
where equality $(a)$ follows from the double-stochasticity, and inequality $(b)$ follows from Jensen's inequality.
By following a similar roadmap as in \cite{wang2022matcha,xiang2023towards,xiang2024efficient,xiang2025empowering},  Lemmas~\ref{lmm: local step perturbation} and~\ref{lmm: spectral norm main}, 
we can bound the term $\|\sum_{q=0}^{t-1} \bm{G}^{(q)} (\prod_{l=q}^{t-1} W^{l} - \allones)\|_{\rm F}^2$. %

The key technical challenge for bounding $\|\sum_{q=0}^{t-1} \bm{Z}^{(q)} (\prod_{l=q}^{t-1}  W^{l} - \allones)\|_{\rm F}^2$ comes from the fact that $\bm{Z}^{q}$ and $W^{q}$ are {\em not independent}.
Observe that, by construction, $\bm{Z}^{q}$ depends on $\calA^q$ and $B_q$, which are both functions of $\calA^q$.
We first upper bound $B_q$ with $C_q = \sum_{i=1}^{m} \|\Delta_{i}^{q}\|_2^2$ because $C_q$ and $(\calA^q, W^{(q)})$ are independent.
It is easy to see that $\prod_{l=q}^{t-1} W^{(l)} - \allones = (W^{(q)} - \allones)( \prod_{l=q+1}^{t-1} W^{(l)} - \allones)$ due to double-stochasticity.
Let $\mathbb{I}_{\calA^q} \triangleq \qth{\indc{1 \in \calA^q}, \cdots, \indc{m \in \calA^q}}$.
We then have $\bm{Z}^{q} =  \frac{B_q}{\gamma}  
\frac{\indc{\abth{\calA^q} > 0} \bm{z}^{q} \mathbb{I}_{\calA^q}
}{\abth{\calA^q} + \indc{\calA^{q} = 0}}$. 
We can show that 
\begin{small}
\begin{align*}
    &\Big \|\frac{\indc{\abth{\calA^q} > 0} \bm{z}^{q} \mathbb{I}_{\calA^q} (W^{(q)} - \allones)}{\abth{\calA^q} +   \indc{\abth{\calA^q} = 0}}\Big\|_{\mathrm{F}}^2 
    \le 
    \norm{\bm{z}^{q}}^2
    \fnorm{W^{(q)} - \allones}^2 
    .
\end{align*}
\end{small}

Using a similar technique to bounding $\|\sum_{q=0}^{t-1} \bm{G}^{(q)} (\prod_{l=q}^{t-1} W^{l} - \allones)\|_{\rm F}^2$ and by~\prettyref{lmm:gradient Bt}, we complete 
bounding $\|\sum_{q=0}^{t-1} \bm{Z}^{(q)} (\prod_{l=q}^{t-1}  W^{l} - \allones)\|_{\rm F}$.
Putting everything together, we obtain the desired consensus error bound in~\prettyref{lmm: consensus} with $\eta \le
(1 - \sqrt{\rho})/(632 s L (\kappa + 1) \sqrt{\beta^2 + 1})$. 
\end{proof}

In the special case where all uplinks are always reliable, \ie, when $c=1$, we have $\rho = 0$, and the consensus error reduces to $0$.
\prettyref{lmm: consensus} says that the spectral norm $\rho$ must be strictly less than $1$ to ensure a valid bound, which is an important ingredient to reach a stationary point in the final convergence.
\prettyref{lmm: ergodicity} guarantees $\rho < 1$.
Let $F^{\ast} \triangleq \min_{\x} F(\x)$.

\begin{theorem}
\label{thm: main}
Suppose Assumptions \ref{ass: threat model},
\ref{ass: 2 smmothness}, %
\ref{ass: bounded variance client-wise}, 
\ref{ass: bounded similarity}, %
and \ref{ass: bounded true gradient} hold. 
Choose a learning rate $\eta $ such that %
$\eta \le
\frac{1 - \sqrt{\rho}}{632 s L m (\kappa + 1) \sqrt{\beta^2 + 1}}$. 
It holds that
\begin{small}
\begin{align*}
&\frac{1}{T}\sum_{t=0}^{T-1}
\expect{\norm{\nabla F(\bar{\x}^t)}^2}
\le
\frac{6(\expect{F(\bar{\x}^{0})} - F^{\ast})}{\eta s T}
+\frac{216 \rho \eta^2 s^2 L^2 \sigma^2}{(1-\sqrt{\rho})^2} \\
&~~~+ 36 \eta s L  \pth{\kappa^2 L^2 + 1}\pth{\sigma^2 + \zeta^2} 
+\frac{73 \eta s L d \sigma_z^2}{c^2 \gamma^2 m}
\pth{\sigma^2 + G_{\max}^2}.
\end{align*}
\end{small}
\end{theorem}
The proof of~\prettyref{thm: main} combines the key intermediate results stated in \prettyref{sec: convergence results} and can be obtained by taking expectation over the remaining randomness, telescoping sum, and term rearrangements.
Putting a specific choice of learning rate $\eta = \sqrt{m}/(L \sqrt{sT})$ in~\prettyref{thm: main}, we have~\prettyref{cor: damped convergence}.
\begin{corollary}
\label{cor: damped convergence}
Suppose Assumptions \ref{ass: threat model},
\ref{ass: 2 smmothness}, %
\ref{ass: bounded variance client-wise}, 
\ref{ass: bounded similarity}, %
and \ref{ass: bounded true gradient} hold. 
Let %
$\eta = \sqrt{m}/(L \sqrt{sT})$, 
where $T \ge 
\pth{
\frac{632 s m (\kappa + 1) \sqrt{\beta^2 + 1}}{(1 - \sqrt{\rho})}
}^2
.$
It holds that
\begin{small}
\begin{align*}
    &\frac{1}{T}\sum_{t=0}^{T-1}
    \expect{\norm{\nabla F(\bar{\x}^t)}^2}
    \le
    \underbrace{\frac{6 L \pth{\Expect[F(\bar{\x}^{0})] - F^\star}}{\sqrt{m s T} } }_{(A)}
   + 
    \underbrace{\frac{216 \rho \sigma^2}{(1 - \sqrt{\rho})^2 } 
    \frac{ms}{T}
    }_{(B)}\\
    &+
    \underbrace{
    36 
    \pth{\kappa^2 L^2 + 1}
    \pth{\sigma^2 + \zeta^2} 
    \sqrt{\frac{ms}{T}}
    }_{(C)}
    +
    \underbrace{\frac{73 d \sigma_z^2  \pth{\sigma^2 + G_{\max}^2}}{c^2 \gamma^2}
    \sqrt{\frac{s}{m T}}
    }_{(D)}
    .
\end{align*}
\end{small}
\end{corollary}

\begin{remark}
    \label{rmk: convergence rate}
    \prettyref{thm: main} and \prettyref{cor: damped convergence} establish our final convergence bound.

    \noindent {\bf On the structures.} $(A)$ characterizes the initialization error, $(B)$ results from the consensus error, $(C)$ is due to noisy stochastic gradients (\prettyref{ass: bounded variance client-wise}) and inter-device gradient heterogeneity (\prettyref{ass: bounded similarity}), and $(D)$ stems from receiver noise.
    Notably, the receiver noise term in~\prettyref{thm: main} exhibits an $\calO (1/m)$ reduction, where simultaneous transmission effectively averages out the receiver noise, consistent with the OTA literature \cite{sery2021over,abrar2025non}.

    \noindent {\bf On the convergence.} 
    \prettyref{cor: damped convergence} shows that $\bar{\x}^t$ converges to a
    stationary point of an {\em unbiased} non-convex objective $F$ at rate $1/\sqrt{T}$, which
    is optimal for any first-order method with access only to stochastic
    gradients~\cite{arjevani2023lower}, rather than to the biased auxiliary
    objective $\tilde F$ in~\eqref{eq:biased convergence aux}. 
    Term~$(A)$ exhibits linear speedup in $m$ and $s$, and term~$(D)$ speeds up with $m$. 
    However, terms~$(B)$ and~$(C)$ ultimately dominate, consistent
    with the literature~\cite{Li2020}. 
    In future work, we will investigate how to achieve linear speedup in both the number of devices $m$ and the number of local steps $s$.

    \noindent {\bf On the role of uplink dynamics.}
    A larger $c$ implies a better uplink reliability and a smaller spectral norm $\rho$.
    Consequently, our bound on $\frac{1}{T} \sum_{t=0}^{T-1} \Expect [\|\nabla F (\bar{\x}^t)\|_2^2]$ would become tighter.
    In the special case when $c=1$, \ie, uplinks are always reliable,~\fedopi~reduces to FedAvg with a receiver noise, and we also have $\rho = 0$.
    Without 
    receiver noise ($\sigma_z^2 = 0$),
    our bound reduces to $\calO(\frac{1}{\sqrt{m s T}} + \sqrt{\frac{ms}{T}} (\sigma^2 + \zeta^2))$ and matches the FedAvg literature \cite{wang2020tackling}.
    When $p_i^t = c$ for all $i \in [m]$ and $t \in [T]$, \ie, uniformly at random uplink availability. Setting $\eta = \sqrt{k/sT}$ in~\prettyref{thm: main}, our convergence rate becomes $\calO (\frac{1}{\sqrt{k s T}} + \frac{\sigma^2 + \zeta^2}{(1 - \sqrt{\rho})^2}\sqrt{\frac{k}{sT}}) $, which introduces a larger variance compared to fully reliable, consistent with literature \cite{yang2021achieving}.
    Note that $\gamma$ also affects $c$.
    For example, we have $p_{i}^{t} = \exp(- \gamma^2/(d \Lambda_{i}^{t}E_s))$ in the special case of Rayleigh fading.
    When the parameters $d$, $\Lambda_{i}^{t}$, and $E_s$ are fixed, a smaller $\gamma$ leads to a greater $c$, yet at the expense of amplifying the receiver noise term $(D)$. 
    Similarly, a greater $\gamma$, which attenuates the receiver noise term $(D)$,  leads to a smaller $c$, \ie, fewer active uplinks.
    In other words, the term $(D)$ reveals the fundamental tradeoff between uplink availability and signal quality.
    We demonstrate such a tradeoff empirically in~\prettyref{fig:impacts of gamma}. A principled choice of $\gamma$ would require explicit knowledge of the statistical CDIs; we defer the design of CDI-agnostic tuning frameworks for $\gamma$ to future work.

\end{remark}

\section{Numerical Experiments}
\label{sec: numerical experiment}
In this section, we evaluate the proposed~\fedopi~against baseline algorithms under time-varying and heterogeneous wireless conditions arising from mobile devices.
\subsection{Experiment Setup}
\label{sec: exp setup}

\noindent{\bf Federated Learning System.}
We consider $m = 30$ devices, each continuing to compute locally despite uncertainties in their uplinks.
Only devices with active uplinks can transmit their updates to the parameter server.
Each device performs $s = 10$ steps of local stochastic gradient descent with a mini-batch size of $32$.
The results are obtained over five repetitions with a total $T=800$ rounds.
We use a multi-layer perceptron as the neural network with ReLU activations.
Hyperparameter settings and network details are deferred to~Appendix E.
We use MNIST \cite{lecun2010mnist} dataset, which admits $10$ classes of images, as the basis for our simulations.
We first partition the datasets according to a Dirichlet distribution parameterized by $0.1$, which creates a highly non-\iid local data distribution across devices. 
Then, we assign partitioned data samples to devices.

\noindent{\bf Wireless Channel Model.}
The devices are uniformly distributed within a cell of radius $r_{\max} = 1500$ meters centered on the parameter server. 
The communication bandwidth is $B = 1$ MHz with carrier frequency $f_c = 2.4 $ GHz, and the transmission power is set to be $P_{\text{tx}} = 0$ dBm. 
The noise power spectral density at the parameter server is $N_0 = -173$ dBm/Hz.
For a device at a distance $D$ from the parameter server, the statistical gain is modeled as
\begin{equation}
  \Lambda(D)=
    10^{-\mathrm{PL}_{\mathrm{ref}}/10}
    \big(\frac{D}{D_0}\big)^{-n},
  \label{eq:path_loss_model}
\end{equation}
where $D$ is the device-to-parameter-server distance, $D_0 = 1$~meter is the
reference distance, $\mathrm{PL}_{\mathrm{ref}} = 50$~dB is the
path loss at $D_0$, and $n = 3.5$ is the path-loss exponent, 
which yields a highly heterogeneous loss profile across devices.
Since devices move according to the mobility model below, we denote $\Lambda_{i}^{t} \triangleq \Lambda(\bm{{\rm p}}_{i}^{t})$ as a function of its position $\bm{{\rm p}}_i^t$ for the time-varying statistical gain of device $i$ in round $t$.
The small-scale fading coefficient is $h_{i,t} \sim \calC\calN(0, \Lambda_{i}^{t})$, drawn independently across devices and rounds.

\noindent{\bf Mobility Model.}
To simulate time-varying channel conditions while preserving statistical
channel heterogeneity across devices, we design a {\em bounded} random waypoint (RWP) model~\cite{hyytia2007random}.
Each device $i$ starts at the initial position $\mathbf{p}_{i}^{0} \in \reals^2$
drawn from the stratified placement described in {\em Wireless Channel Model},
and a personal territory of radius $r_{\text{local}}=50$ m, centered on
$\mathbf{p}_{i}^{0}$, is fixed for the duration of training.
At the start of each leg, device $i$ draws a target waypoint
$\mathbf{w}_i$ uniformly at random over its personal territory disc
and a travel speed $v_i \sim \mathcal{U}[v_{\min}, v_{\max}]$.
It then moves in a straight line toward $\mathbf{w}_i$ at speed $v_i$ for a duration of $\tau$
until the waypoint is reached, at which point a new waypoint and speed are drawn immediately, and the next leg begins.
To draw $\mathbf{w}_i$ uniformly over the personal territory disc of
radius $r_{\text{local}}$, one samples an angle $\phi \sim \calU [0, 2\pi]$
and a radius via the inverse-CDF of the marginal PDF
$f(r) = 2r / r_{\text{local}}^2$,
giving waypoint coordinates
$\mathbf{w}_i = \mathbf{p}_{i}^{0} + (r\cos\phi,\, r\sin\phi)$,
clipped to the cell boundary.
At the end of round $t$, device $i$ has traveled at most
$\Delta_i = v_i \tau$ along its current leg, and its updated position is
\begin{small}
\begin{equation}
  \mathbf{p}_{i}^{t+1} = \mathbf{p}_{i}^{t}
    + \min\pth{1,
        \frac{\Delta_i}{\norm{\mathbf{w}_i - \mathbf{p}_{i}^{t}}}}
      (\mathbf{w}_i - \mathbf{p}_{i}^{t}),
  \label{eq:position_update}
\end{equation}
\end{small}
where $\tau = 8$ seconds.

\noindent{\bf Baseline algorithms.}
We compare the proposed~\fedopi~with five baseline algorithms.
(i) The ideal FedAvg algorithm \cite{mcmahan2017communication} aggregates each device's update without any receiver noise,
(ii) the SCA algorithm \cite{abrar2025non},
(iii) the BB-interior algorithm \cite{zhu2019broadband},
(iv) the BB-alternative algorithm \cite{zhu2019broadband},
and (v) the vanilla OTA \cite{yang2020federated}.
\prettyref{tab:cdi knowledge} specifies the CDI knowledge of each algorithm.
Details of algorithm implementations are deferred to~Appendix E.

\subsection{Experiment Results}
\label{sec: exp results}
\begin{figure}[!t]
    \centering
    \begin{subfigure}[b]{.85\linewidth}
    \includegraphics[width=\textwidth,trim=0.1cm 0 0 0, clip]{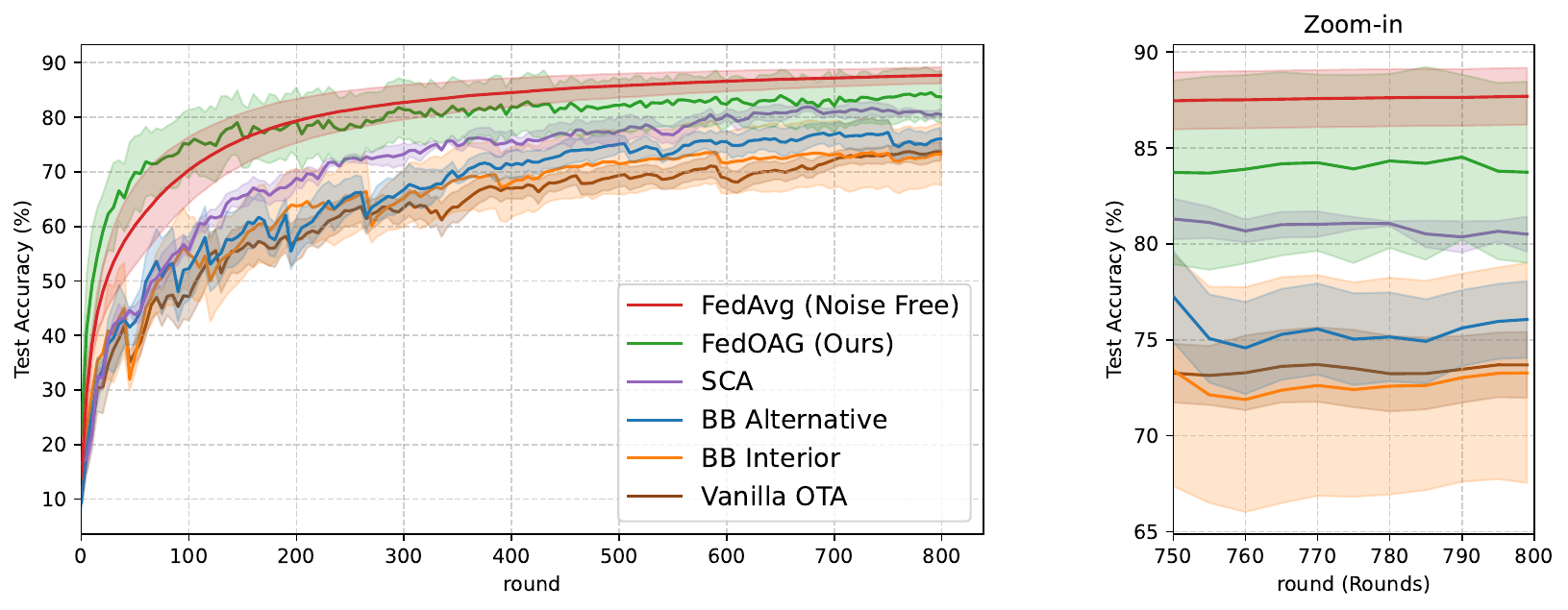}    
    \caption{Test accuracy for stationary statistical CDI dynamics.}
    \label{fig: stationary}
    \end{subfigure}

    \begin{subfigure}[b]{.85\linewidth}
    \includegraphics[width=\textwidth,trim=0.1cm 0 0 0, clip]{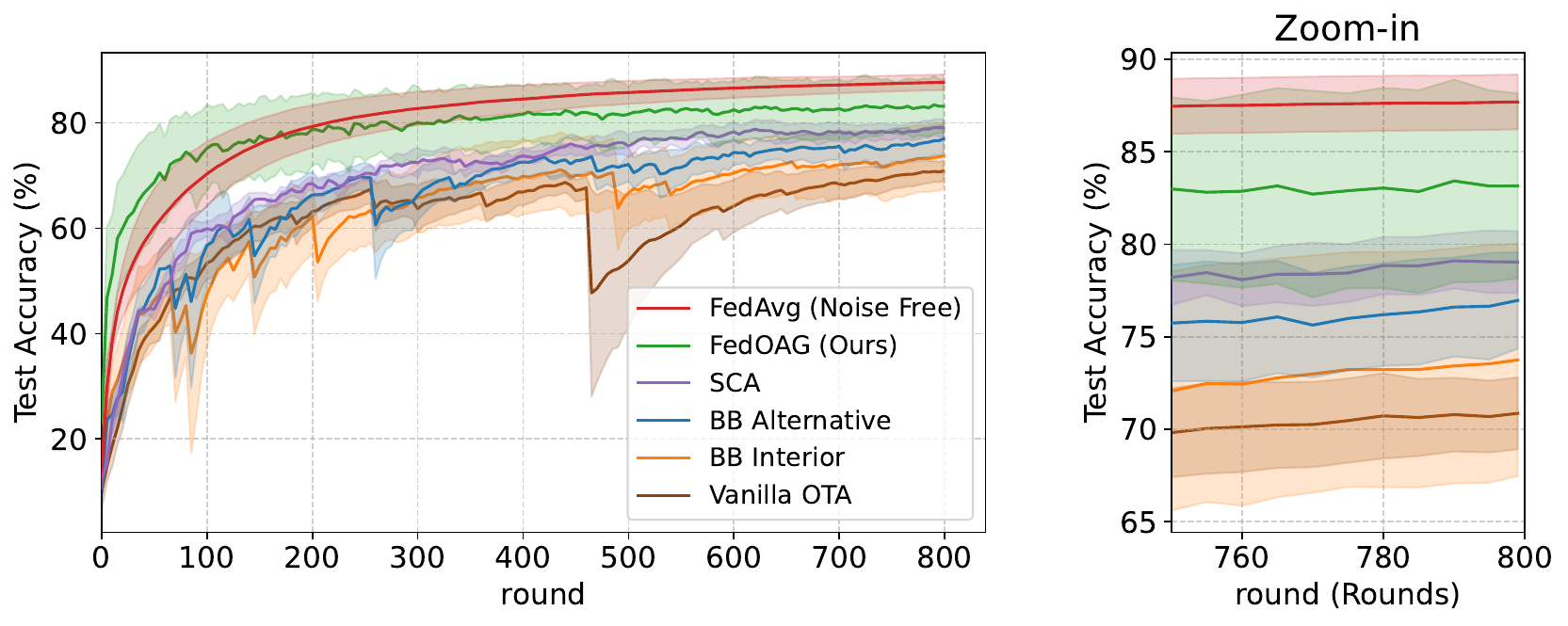}    
    \caption{Test accuracy for non-stationary statistical CDI dynamics.}
    \label{fig: non-stationary}
    \end{subfigure}

    \begin{subfigure}[b]{.85\linewidth}
    \includegraphics[width=\textwidth,trim=0.1cm 0 0 0, clip]{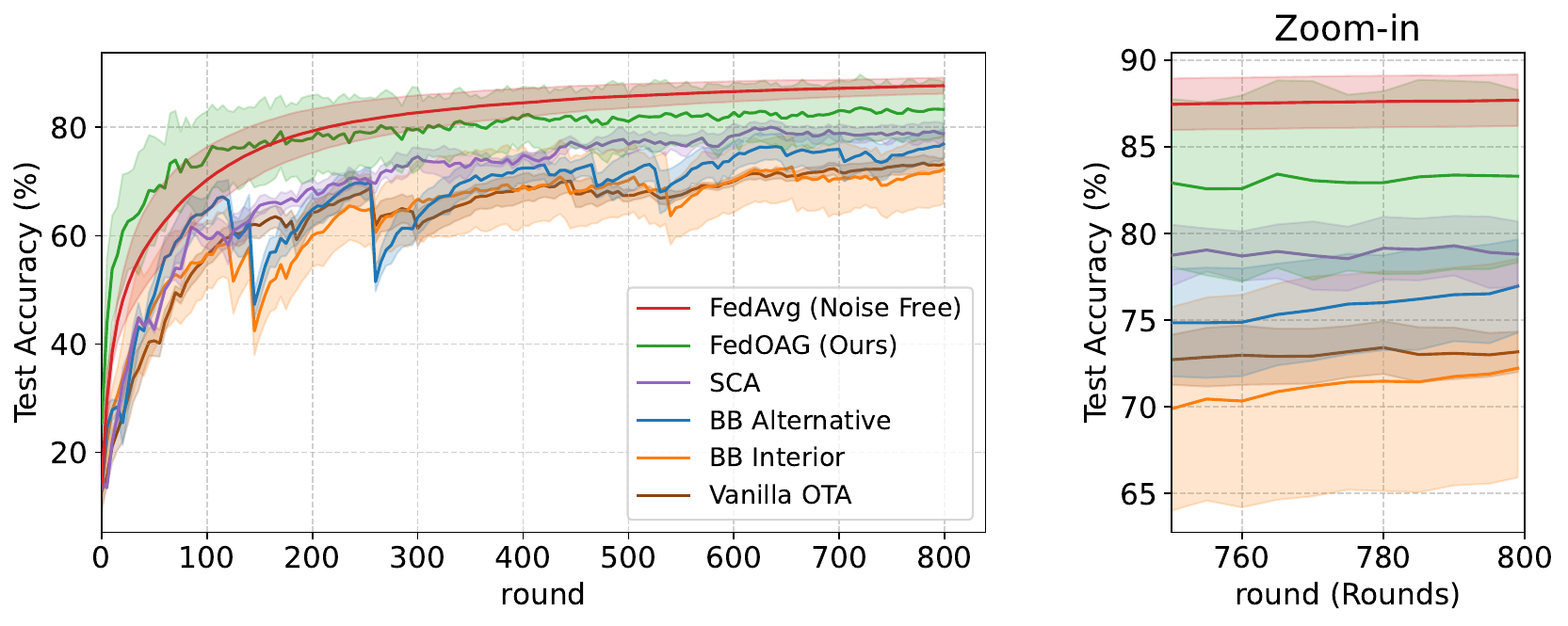}    
    \caption{Test accuracy for mixed non-stationary statistical CDI dynamics.}
    \label{fig: mixed non-stationary}
    \end{subfigure}

    \caption{\small Comparisons of different OTA federated learning algorithms on the MNIST dataset with $m = 30$ devices. 
    The results are obtained over five repetitions with solid curves depicting mean test accuracy, and the shaded area denotes standard deviation.
    }
    \label{fig:final results}
\end{figure}
\begin{figure}[!t]
    \centering
    \includegraphics[width=.85\linewidth]{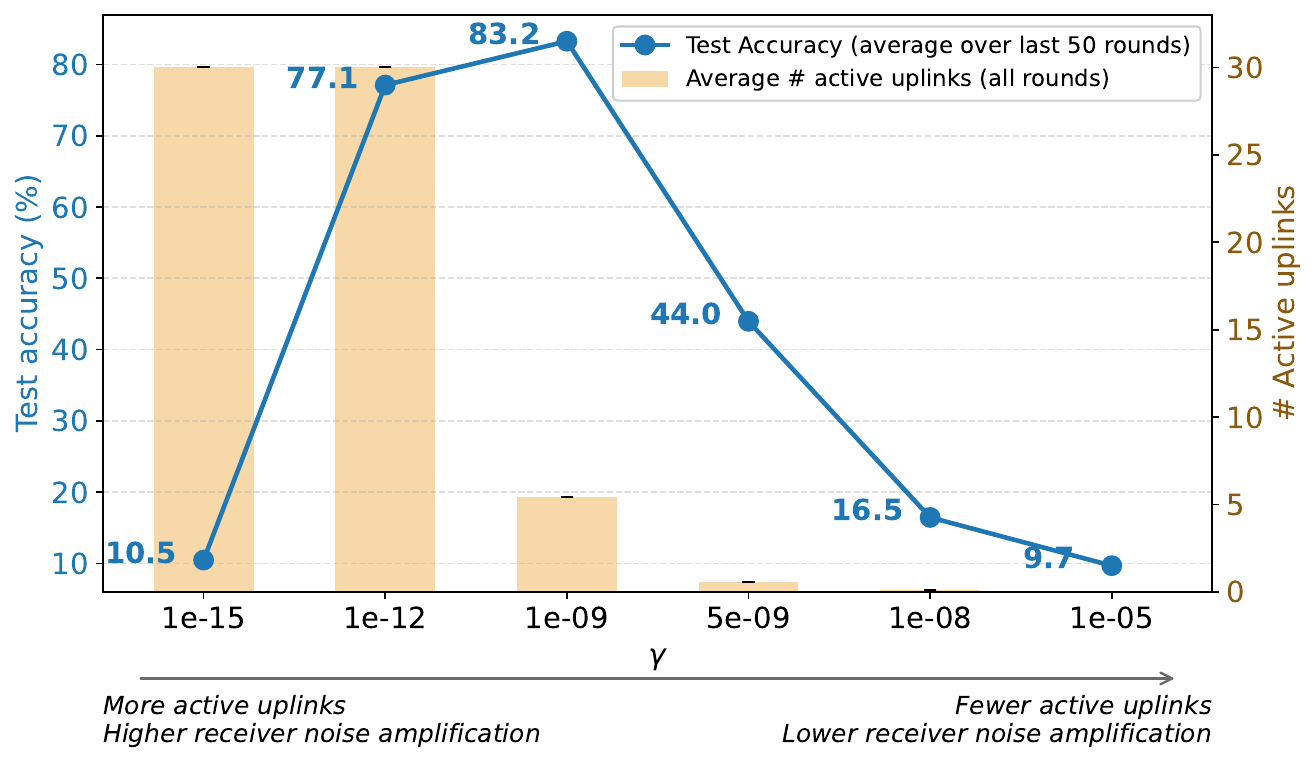}
    \caption{\small
    Comparison of \fedopi~under different $\gamma$ on MNIST with $m=30$ devices
    in the stationary regime, averaged over five repetitions. 
    The blue curve (left axis) shows test accuracy over the final $50$ rounds; yellow bars (right axis) show the average number of active devices per round. Smaller $\gamma$ lowers the truncation threshold, yielding more active uplinks but greater receiver noise amplification, and vice versa. 
    $\gamma$ values on the $x$-axis are uniformly spaced for readability.
    }
    \label{fig:impacts of gamma}
\end{figure}
We consider three mobility regimes.
(i) Stationary:
$v_{\min} = v_{\max} = 0$ m/s. In this regime, devices do not move throughout training. statistical gains $\Lambda_{i}^{t} = \Lambda_{i}^{0}$ for all $i \in [m]$ and $t \in [T]$.
(ii) Non-stationary: Following the mobility classifications defined in \cite[Table 7.2-3]{3gpp38901} and \cite{iturm1225}, we consider the speed range $v_{\min} = 0$ m/s, $v_{\max} = 2.5$ m/s $=9$ km/h, which covers walking to cycling speeds.
(iii) Mixed non-stationary: Half of the devices remain stationary throughout training. The other half are pedestrian mobile with $v_{\min} = 0$ m/s, and $v_{\max} = 2.5$ m/s.

Ideal FedAvg, which aggregates updates without receiver noise, serves as the
benchmark. \prettyref{fig:final results} shows that \fedopi~achieves the
highest test accuracy among all distributed OTA baselines, maintaining
around 83\% across mobility regimes. The SCA algorithm reaches
about 81\% in the stationary regime but drops below 80\% under
non-stationarity. 
Although the SCA tunes its pre-scalars through a bias-variance tradeoff, its convergence remains biased, whereas \fedopi~converges asymptotically to a stationary point of the non-convex objective.
Figs.~\ref{fig: non-stationary} and~\ref{fig: mixed non-stationary} further
show that vanilla OTA, BB-interior, and BB-alternative fluctuate heavily
under mobility, while \fedopi~remains stable. 
The BB methods admit only a subset of devices per round, yielding a highly dynamic aggregation population that injects variance into the trajectory. Vanilla OTA, by forcing every device to transmit, is dominated by the weakest
channel, which disproportionally amplifies receiver noise. 
\fedopi's shaded band is comparable to the SCA and BB-alternative but centered at a higher mean test accuracy, while vanilla OTA and BB-interior show markedly wider dispersion. Finally,~\fedopi~requires only local small-scale CDI at each device, avoiding the overhead and latency of global CDI acquisition at the parameter server.

We can see from
\prettyref{fig:final results} that \fedopi~exhibits no
significant slowdown, confirming the mild impact of stale updates
(Lemma 7 in Appendix).
In fact,~\fedopi~achieves the fastest convergence among the OTA baselines. 
This is partly enabled by the ability to tune $\gamma$ to balance uplink availability against receiver noise.
\prettyref{fig:impacts of gamma} illustrates this tradeoff: smaller $\gamma$
activates more uplinks but degrades utility due to significantly amplified receiver noise, while larger $\gamma$
suppresses noise at the cost of fewer active uplinks. 
Test accuracy peaks at an intermediate $\gamma = 10^{-9}$, where the two effects are jointly balanced.

\section{Conclusion}
\label{sec: conclusion}
In this paper, we have shown that heterogeneous and time-varying statistical CDIs induce highly uncertain uplink availability probabilities $ p_i^t$.
To tackle this, we have proposed~\fedopi, which converges to a stationary point of an unbiased non-convex objective 
via truncated transmission with gradient normalization
and a postponed global broadcast.
Notably, the algorithm achieves convergence without knowing the underlying true dynamics in $p_i^t$'s or the fading models.
Future work will investigate how to improve storage efficiency, 
avoid the need to save checkpoints, 
relax~\prettyref{ass: bounded true gradient}, 
and develop a CDI-agnostic tuning framework for $\gamma$.

\section*{Acknowledgements}
We gratefully acknowledge the support from the NSF CAREER award under grants 2129015 and 2340482. The views and conclusions contained in this document are those of the authors and should not be interpreted as representing the official policies, either expressed or implied, of the NSF or the U.S. Government. The U.S. Government is authorized to reproduce and distribute reprints for Government purposes notwithstanding any copyright notation herein.

\bibliographystyle{ACM-Reference-Format}
\bibliography{bib,ota_fl}

\clearpage
\newpage
\onecolumn
\appendix

\begin{figure*}[!t]
\begin{minipage}{\textwidth}
\begin{center}
{\LARGE Appendix for
{\bf  ``Over-the-Air Federated Learning in Heterogeneous Mobile Wireless Networks''}}

\bigskip
{\large
Ming Xiang,
Nicolò Michelusi,
Yonina C. Eldar,
Lili Su
}

\noindent\hrulefill
\end{center}
\end{minipage}
\end{figure*}

\section{Additional Related Work}
\label{app: dynamic device availability}

\noindent {\bf Dynamic device availability.}
Such dynamic client availability is often handled either via proactively server-side scheduling \cite{mcmahan2017communication, cho2022towards} or 
by modeling uncontrollable device availability through homogeneous Markovian dynamics \cite{ribero2022federated}, 
regularized participation \cite{wang2022,crawshaw2024federated},
and time-invariant but heterogeneous available probabilities \cite{wang2023lightweight,ying2025exact}.
Xiang et al. develop resilient algorithms for heterogeneous and time-varying device availability in \cite{xiang2023towards,xiang2024efficient,xiang2025empowering} in standard digital networks, where device updates are transmitted and decoded individually without errors.
The OTA federated learning, however, relies on the superposition property of analog waveforms, the absence of individual device updates fundamentally alters the aggregation protocol.
For example, the algorithms proposed in \cite{gu2021fast,jhunjhunwala2022fedvarp,yan2023federated}, despite their superior numerical performance, require the parameter server to access each device's most recent $d$-dimensional individual multi-step gradients. 
Unless we can have a separate procedure to extract and deliver those memorized updates, direct application of these algorithms to OTA remains infeasible.
Our algorithm is specifically designed for noisy OTA aggregation and 
can correct the bias arising from heterogeneous and time-varying uplink $p_i^t$'s in wireless communication channels.

\noindent {\bf Asynchronous federated learning.}
Another related line of research is asynchronous federated learning \cite{xie2019asynchronous,nguyen2022federated,toghani2022unbounded,koloskova2022sharper}, in which devices experience message transmission delays, and the reported updates may be stale.
However, the proposed methods therein assume that the uplinks of all devices are either always available or uniformly available at random, which is infeasible in practice.

\section{Proof of~\prettyref{lmm: ergodicity}}
We include the proof here for completeness.
For ease of exposition, we drop time index $t$ in this proof.
We first get the explicit expression for $\expect{W^2_{jj^{\prime}}\mid \calA \neq \emptyset}$. 
Suppose that $\calA\not=\emptyset$, 
we have
\begin{align*}
(W^2)_{jj^{\prime}} & = \sum_{k=1}^m W_{jk}W_{j^{\prime}k} 
= W_{jj}W_{j^{\prime}j} + W_{jj^{\prime}}W_{j^{\prime}j^{\prime}} 
\indc{j \not = j^{\prime}}
+ \sum_{k\in [m]\setminus \{j, j^{\prime}\}} W_{jk}W_{j^{\prime}k}.   
\end{align*}
When $k\not=j$ and $k\not=j^{\prime}$ by Eq.\,\eqref{eq: gossiping matrix}, we have 
\begin{align*}
W_{jk}W_{j^{\prime}k} & = \frac{1}{|\calA|^2} \indc{j\in \calA} \indc{j^{\prime}\in \calA}\indc{k\in \calA}.     
\end{align*}
In addition, we have $W_{jj}W_{j^{\prime}j}  = \frac{1}{|\calA|^2} \indc{j\in \calA}\indc{j^{\prime}\in \calA},$ and $W_{j^{\prime}j^{\prime}}W_{jj^{\prime}} = \frac{1}{|\calA|^2} \indc{j\in \calA}\indc{j^{\prime}\in \calA}$.  
Thus, 
\begin{itemize}
\item 
For  $j \neq j^\prime$, we have
\begin{align*}
&(W^2)_{jj^{\prime}}
= \sum_{k=1}^m W_{jk}W_{j^{\prime}k} 
= \frac{1}{|\calA|}\indc{j\in \calA} \indc{j^{\prime}\in \calA};
\end{align*}
\item
For $j = j^\prime$, we have
\begin{align*}
&(W^2)_{jj} = \frac{1}{|\calA|}\indc{j\in \calA} + \pth{1-\indc{j\in \calA}}.
\end{align*}
\end{itemize}
\noindent{\em (i) The general case where $p_i^t \ge c$.}
In the special case where $\calA = \emptyset$, we simply have $W = \identity$ by the algorithmic clauses.
Therefore,
$\expect{W_{j j^\prime} \mid \calA = \emptyset} \ge 0$ holds for any pair of $j,j^\prime  \in [m]$.
It follows, by the law of total expectation and for all $j,j^\prime \in [m]$, that
\begin{align}
\expect{W_{j j^\prime}}
&=
\expect{W_{j j^\prime} \mid \calA = \emptyset}
\prob{\calA = \emptyset} 
+
\expect{W_{j j^\prime} \mid \calA \neq \emptyset}
\prob{\calA \neq \emptyset} 
\ge
\expect{W_{j j^\prime} \mid \calA \neq \emptyset}
\prob{\calA \neq \emptyset}.
\label{eq: W matrix total expectation}
\end{align}
\begin{itemize}
\item 
For $j \neq j^\prime$, it holds that
\begin{align*}
\expect{(W^2)_{jj^{\prime}}\mid \calA \neq \emptyset} 
&=
\expect{\frac{1}{|\calA|}\indc{j\in \calA} \indc{j^{\prime}\in \calA} \Big| \calA \neq \emptyset} 
\overset{(\rma)}{\ge} 
\expect{\frac{1}{m}\indc{j\in \calA} \indc{j^{\prime}\in \calA} \Big| \calA \neq \emptyset} 
= 
\frac{p_j p_{j^{\prime}}}{m} 
\ge \frac{c^2}{m},
\end{align*}
where $(\rma)$ holds because $\abth{\calA} \le m$
;
\item
For $j = j^\prime$, it holds that
\begin{align*}
&\expect{(W^2)_{jj}\mid \calA\neq\emptyset} 
= 
\expect{\frac{1}{|\calA|}\indc{j\in \calA} + \pth{1-\indc{j\in \calA}} \Big| \calA \neq \emptyset} 
\ge
\expect{\frac{1}{m}\qth{\indc{j\in \calA} + \pth{1-\indc{j\in \calA}}} \Big| \calA \neq \emptyset} 
=
\frac{1}{m}
\ge
\frac{c^2}{m}.
\end{align*}
\end{itemize}
Recall that $M = \expect{W^2}$.
Next, we show that each element of $M$ is lower bounded.
\begin{align*}
M_{jj^{\prime}}
\ge 
\expect{(W^2)_{j j^\prime} 
\mid \calA \neq \emptyset}
\prob{\calA \neq \emptyset}
\ge \frac{c^2}{m} \qth{1-\pth{1-c}^m}.
\end{align*} 
We note that $\rho (t) = \lambda_2 (M)$,
where $\lambda_2$ is the second largest eigenvalue of %
matrix $M$. 
A Markov chain with $M$ as the transition matrix is ergodic as the chain is (1) {\it irreducible}: $M_{j j^\prime}\ge  \frac{c^2}{m}\qth{1-\pth{1-c}^m}>0$ for $j,j^\prime \in[m]$ and (2) {\it aperiodic} (it has self-loops).  
In addition, $W$ matrix is by definition doubly-stochastic. 
Hence, $M$ has a uniform stationary distribution
$
\pi = \frac{1}{m}\Indc^\top.
$
Furthermore, the irreducible Markov chain is reversible since 
it holds
for all the states that
$
\pi_i M_{ij} = \pi_j M_{ji}.
$
The conductance of a reversible Markov chain \cite{jerrum1988conductance}
with  a transition matrix $M$ can be bounded by
\begin{align*}
\Phi(M) 
&= \min_{\sum_{i\in\calS} \pi_i \le \frac{1}{2}} \frac{\pi_i \sum_{i\in\calS, j\notin \calS} M_{ij}}{\sum_{i\in \calS} \pi_i}
\ge \frac{\pth{\frac{c}{m}}^2\qth{1-\pth{1-c}^m}\abth{\calS}\abth{\bar{\calS}}}{\frac{\abth{\calS}}{m}} = \frac{c^2\qth{1-\pth{1-c}^m}}{m} \abth{\bar{\calS}},
\end{align*}
where
$
\abth{\bar{\calS}} = m - \abth{\calS} \ge \frac{m}{2}.
$
From Cheeger's inequality, we know that
$
\frac{1 - \lambda_2}{2} \le \Phi(M) \le \sqrt{2 \pth{1-\lambda_2}}.
$
Finally, we have
\begin{align*}
\Phi(M) 
&\ge \frac{c^2\qth{1-\pth{1-c}^m}}{m} \abth{\bar{\calS}} \ge \frac{c^2\qth{1-\pth{1-c}^m}}{2}.
\end{align*}
Thus,
$
\rho(t) = \lambda_2 \le 1 - \frac{\Phi^2\pth{M}}{2} \le 1 - \frac{c^4\qth{1-\pth{1-c}^m}^2}{8}.
$

\section{Convergence Analysis}
Our analysis leverages compact matrix representations of several key quantities, which are especially useful for characterizing the consensus error.  
Formally, define  
\begin{align*}
\bm{X}^{(t)} & = \qth{\x_1^t, \cdots, \x_m^t};   \\
\bm{Z}^{(t)} & = \frac{B_t}{\gamma}
\frac{\indc{\abth{\calA^t} > 0}}{{{\abth{\calA^t} + \indc{\calA^{t} = 0}}}}
\qth{\bm{z}^t \indc{1 \in \calA^t}, \cdots, \bm{z}^t \indc{m \in \calA^t}};   \\
\bm{G}_0^{(t)} &= \qth{s \nabla \ell_1(\x_1^{(t,0)}), \cdots, s \nabla \ell_m(\x_m^{(t,0)})};\\
\bm{G}^{(t)} & = \qth{\sum_{r=0}^{s-1}\nabla \ell_1(\x_1^{(t,r)}), \cdots,  \sum_{r=0}^{s-1}\nabla \ell_m(\x_m^{(t,r)})} ;\\
\nabla \bm{F}^{(t)} & = \qth{\nabla F_1(\x_1^t), \cdots,  \nabla F_m(\x_m^t)},
\end{align*}
where $\bm{X}^{(t)}$ is a concatenated model parameter matrix for all devices at the {\em beginning} of round $t$,
$\bm{G}_0^{(t)}$ is an initial local stochastic gradient matrix,
$\bm{G}^{(t)}$ is a local stochastic gradient accumulation,
$\nabla \bm{F}^{(t)}$ is a true local gradient matrix,
and $\bm{Z}^{(t)}$ is a scaled receiver noise matrix.
When $\calA^t \not = \emptyset$, the receiver noise is rescaled by $B_t / (\gamma \abth{\calA^t})$.
When $\calA^t = \emptyset$, it is easy to check that we simply have $\bm{Z}^{t} = 0 \identity$.

\subsection{\bf Proof of~\prettyref{lmm: descent lemma}}
Recall that
\begin{align*}
\bar{\x}^t \triangleq \frac{1}{m}\sum_{i=1}^m \x_i^t. 
\end{align*} 
Unroll one iteration of $\bar{\x}^{t+1}$, we have 
\begin{align*}
    \bar{\x}^{t+1} - \bar{\x}^{t}
    &=
    \frac{1}{m} 
    \sum_{i \in \calA^t}
    \pth{
    \frac{1}{\abth{\calA^t}} 
    \pth{\frac{\eta B_t \bm{z}^t}{\gamma}
    +
    \sum_{i\in \calA^t} \tilde \x_i^{t} 
    }
    - \x_i^t} 
    +
    \frac{1}{m} 
    \sum_{i \not \in \calA^t}
    \pth{
    \tilde \x_i^{t}
    - \x_i^t
    } 
    =
    -
    \frac{\eta}{m} 
    \sum_{i=1}^m
    \sum_{r=0}^{s-1}
    \nabla \ell_{i}(\x_{i}^{(t,r)})
    +
    \frac{\eta }{m} \frac{B_t \bm{z}^t}{\gamma}
    .
\end{align*}

By~\prettyref{ass: 2 smmothness},
we have 
\begin{align*}
F(\bar{\x}^{t+1})  -  F(\bar{\x}^{t}) 
&\le 
\iprod{\nabla F(\bar{\x}^{t})}{\bar{\x}^{t+1} - \bar{\x}^{t}} + \frac{L}{2}\norm{\bar{\x}^{t+1}- \bar{\x}^{t}}^2 
= 
\iprod{\nabla F(\bar{\x}^{t})}{- \frac{\eta }{m} \bm{G}^{(t)} \bm{1}} +  
\iprod{\nabla F(\bar{\x}^{t})}{\frac{\eta}{m} \frac{B_t}{\gamma} \bm{z}^{t}} 
+
\frac{L}{2}\norm{\frac{\eta}{m} \bm{G}^{(t)} \bm{1} - 
\frac{\eta B_t \bm{z}^t}{\gamma m} 
}^2.  
\end{align*}
Taking expectations conditional on filtration $\calF^t$, we have 
\begin{align*}
&\expect{F(\bar{\x}^{t+1})  -  F(\bar{\x}^{t}) \mid \calF^{t}} 
\le \expect{\iprod{\nabla F(\bar{\x}^{t})}{- \frac{ \eta}{m} \bm{G}^{(t)} \bm{1}} +  
\iprod{\nabla F(\bar{\x}^{t})}{\frac{\eta}{m} \frac{B_t}{\gamma} \bm{z}^{t}} 
+
\frac{L}{2}\norm{\frac{\eta}{m} \bm{G}^{(t)} \bm{1} - 
\frac{\eta B_t \bm{z}^t}{\gamma m}
}^2 \mid \calF^{t}}. 
\end{align*}
We know, by zero-mean Gaussian noise and independence, that 
\begin{align*}
    \expect{\iprod{\nabla F(\bar{\x}^{t})}{\frac{1}{\gamma m} B_t \bm{z}^{t}} \mid \calF^t} 
    &=
    \iprod{\nabla F(\bar{\x}^{t})}{\frac{1}{\gamma m} \expect{B_t \bm{z}^{t} \mid \calF^t} } 
    =
    \iprod{\nabla F(\bar{\x}^{t})}{\frac{1}{\gamma m} \expect{B_t \mid \calF^t} \expect{\bm{z}^{t} \mid \calF^t} } 
     = 0,
\end{align*}
where the last equality holds because $B_t$ and $\bm{z}^t$ are independent.
It follows that
\begin{align*}
\expect{\frac{L}{2}\norm{\frac{\eta}{m} \bm{G}^{(t)} \bm{1} - 
\frac{\eta}{\gamma m} B_t \bm{z}^t
}^2 \mid \calF^{t}}&\le
\expect{L \pth{\norm{\frac{\eta}{m} \bm{G}^{(t)} \bm{1}}^2 +\norm{\frac{\eta B_t \bm{z}^t}{\gamma m}}^2} \Big | \calF^{t}} 
\le
L 
\expect{\norm{\frac{\eta}{m} \bm{G}^{(t)} \bm{1}}^2 \mid \calF^{t}} 
+
\frac{\eta^2 L d \sigma_z^2}{\gamma^2 m^2}
\expect{B_t^2 \mid \calF^t}. 
\end{align*}
For ease of notations, we abbreviate $\nabla \ell_i(\x_i^{\pth{t,k}})$ as $\nabla \ell_i^{\pth{t,k}}.$

\paragraph{\em (i) Bounding
$\mathbb{E}[\iprod{\nabla F(\bar{\x}^{t})}{- \frac{\eta}{m} \nabla \bm{G}^{(t)} \bm{1}} \mid \calF^{t}]$.
}
\begin{align*}
\expect{\iprod{\nabla F(\bar{\x}^{t})}{- \frac{\eta}{m} \bm{G}^{(t)} \bm{1}} \mid \calF^{t}} 
&= - \frac{\eta}{m} \expect{\iprod{\nabla F(\bar{\x}^{t})}{\sum_{i=1}^m \sum_{k=0}^{s-1} \nabla \ell_i^{(t,k)}} \mid \calF^{t}}    \\
& = \underbrace{- \frac{s\eta}{m} \iprod{\nabla F(\bar{\x}^{t})}{\nabla \bm{F}^{(t)}\bm{1}}}_{(\rmA)} 
+ \underbrace{\expect{\frac{\eta}{m} \iprod{\nabla F(\bar{\x}^{t})}{\sum_{i=1}^m s \nabla \ell_i^{(t,0)} - \sum_{k=0}^{s-1} \nabla \ell_i^{(t,k)} }\mid \calF^{t}}}_{(\rmB)}.  
\end{align*}

Term $(\rmA)$ can be bounded as 
\begin{align*}
- s\eta \iprod{\nabla F(\bar{\x}^{t})}{ \frac{1}{m} \nabla \bm{F}^{(t)} \Indc} 
 &=  -\frac{s\eta}{2} \norm{\nabla F(\bar{\x}^{t})}^2 
+ \frac{s\eta }{2} \norm{\nabla F(\bar{\x}^{t}) - \frac{1}{m}\nabla \bm{F}^{(t)}\bm{1}}^2 
- \frac{s\eta }{2} \norm{\frac{1}{m}\nabla \bm{F}^{(t)}\bm{1}}^2\\
& \le  -\frac{s\eta }{2} \norm{\nabla F(\bar{\x}^{t})}^2 - \frac{s\eta}{2} \norm{\frac{1}{m}\nabla \bm{F}^{(t)}\bm{1}}^2
+ \frac{s\eta L^2 }{2m}\sum_{i=1}^m \norm{\bar{\x}^{t} - \x_i^t}^2.   
\end{align*}

For term (B), we have
\begin{align*}
\expect{\frac{\eta }{m} \iprod{\nabla F(\bar{\x}^{t})}{\sum_{i=1}^m s \nabla \ell_i^{(t,0)} - \sum_{k=0}^{s-1} \nabla \ell_i^{(t,k)}} \mid \calF^{t}} 
& = \frac{\eta}{m} \sum_{i=1}^m\iprod{\nabla F(\bar{\x}^{t})}{ \expect{ s \nabla \ell_i^{(t,0)} - \sum_{k=0}^{s-1} \nabla \ell_i^{(t,k)}\mid \calF^{t}}} \\
& \overset{(\rma)}{\le} \frac{\eta^2 s^2}{2} \norm{\nabla F(\bar{x}^{t})}^2 
+ \underbrace{\frac{1}{2m s^2} \sum_{i=1}^m \expect{\norm{s \nabla \ell_i^{(t,0)} - \sum_{k=0}^{s-1} \nabla \ell_i^{(t,k)}}^2 \Big | ~\calF^{t}}}_{(\rmB.1)}, 
\end{align*}
where inequality $(\rma)$ holds because of Young's inequality.
By Lemma \ref{lmm: local step perturbation}, we bound term $(\rmB.1)$ as follows  
\begin{align*}
\frac{1}{2m s^2} \sum_{i=1}^m \expect{\norm{s \nabla \ell_i^{(t,0)} - \sum_{k=0}^{s-1} \nabla \ell_i^{(t,k)}}^2 \mid \calF^{t}} 
& \overset{(\rmb)}{\le} \frac{1}{2m s^2} \sum_{i=1}^m \expect{ \kappa^2 \eta^2 \binom{s}{2}^2 L^2\norm{\nabla \ell_{i}^{\pth{t,0}}}^2 \mid \calF^{t}} \\
& = \frac{\kappa^2 \eta^2 \binom{s}{2}^2 L^2}{2m s^2} \sum_{i=1}^m \expect{\norm{\nabla \ell_{i}^{\pth{t,0}} - \nabla F_i(\x_i^t) + \nabla F_i(\x_i^t)}^2 \mid \calF^{t}} \\
& \overset{(\rmc)}{\le} \kappa^2 \eta^2 L^2 \sigma^2\frac{s^2}{4} + \frac{\kappa^2 \eta^2 s^2 L^2}{4m} \sum_{i=1}^m\norm{\nabla F_i(\x_i^t)}^2\\
& \le \kappa^2 \eta^2 s^2 L^2 \frac{L^2}{m} \sum_{i=1}^m \norm{\x_i^t - \bar{\x}^t}^2 
+ \kappa^2 \eta^2 s^2 L^2 (\zeta^2 +\sigma^2)
+ \kappa^2 \eta^2 s^2 L^2\pth{\beta^2 + 1}\norm{\nabla F(\bar{\x}^t)}^2, 
\end{align*} 
where 
inequality $(\rmb)$ follows from Lemma~\ref{lmm: local step perturbation},
inequality $(\rmc)$ follows from Assumption \ref{ass: bounded variance client-wise}, 
and the last inequality holds because of Proposition \ref{prop: average gradient to global gradient}. 
Combing the bounds of terms $(\rmA)$ and $(\rmB)$, we get 
\begin{align}
\label{eq: bound 1}
\nonumber
\expect{\iprod{\nabla F(\bar{\x}^{t})}{- \frac{\eta}{m} \bm{G}^{(t)} \bm{1}} \mid \calF^{t}} 
&\le- \qth{\frac{s\eta}{2} - \frac{\eta^2 s^2}{2} - \kappa^2 \eta^2 s^2 L^2\pth{\beta^2 + 1}} \norm{\nabla F(\bar{\x}^t)}^2 
 - \frac{s\eta}{2} \norm{\frac{1}{m}\nabla \bm{F}^{(t)}\bm{1}}^2 \\
&~~~ + \kappa^2 \eta^2 s^2 L^2 (\zeta^2 +\sigma^2)
+  \pth{\frac{s\eta L^2 }{2m} + \kappa^2 \eta^2 s^2 L^2 \frac{L^2}{m}} \sum_{i=1}^m \norm{\bar{\x}^{t} - \x_i^t}^2. 
\end{align}
\paragraph{\em (ii) Bounding $\expect{\norm{\frac{1}{m} \bm{G}^{(t)} \bm{1}}^2 \mid \calF^{t}}$.}
By adding and subtracting, we get
\begin{align*}
&\norm{\frac{1}{m} \bm{G}^{(t)} \bm{1}}^2  =  \norm{\frac{1}{m} \sum_{i=1}^m \sum_{k=0}^{s-1} \nabla \ell_i^{(t,k)}}^2 
\le 2 \underbrace{\norm{\frac{1}{m} \sum_{i=1}^m \sum_{k=0}^{s-1} \pth{\nabla \ell_i^{(t,k)} - \nabla \ell_i^{(t,0)}}}^2}_{(\rmC)} + 2 \underbrace{\norm{\frac{s}{m} \sum_{i=1}^m   \nabla \ell_i^{(t,0)}}^2}_{(\rmD)}. 
\end{align*}    
\noindent For term $(\rmC)$, %
by Lemma \ref{lmm: local step perturbation}, we have
\begin{align*}
\expect{\norm{\frac{1}{m} \sum_{i=1}^m \sum_{k=0}^{s-1} \pth{\nabla \ell_i^{(t,k)} - \nabla \ell_i^{(t,0)}}}^2 \mid \calF^t}
&\le  \frac{\kappa^2 \eta^2 s^4 L^2}{4m}\sum_{i=1}^m \expect{\|\nabla \ell_i^{(t,0)}\|_2^2 \mid \calF^t}\\
&\le \frac{\kappa^2 \eta^2 s^4 L^2}{2m} (\sum_{i=1}^m \expect{\norm{\nabla \ell_i^{(t,0)} - \nabla F_i(\x_i^t)}^2 \mid \calF^t} + \sum_{i=1}^m \norm{\nabla F_i(\x_i^t)}^2) \\
& \overset{(\rmd)}{\le}  \frac{\kappa^2 \eta^2 s^4 L^2\sigma^2}{2} + \frac{\kappa^2 \eta^2 s^4 L^2}{2m} \sum_{i=1}^m \norm{\nabla F_i(\x_i^t)}^2,
\end{align*}
where inequality $(\rmd)$ holds because of Assumption \ref{ass: bounded variance client-wise}.
For term $(\rmD)$,  by Assumption \ref{ass: bounded variance client-wise}, we likewise have 
\begin{align*}
&\frac{s^2}{m^2} 
\expect{\|{\sum_{i=1}^m   \nabla \ell_i^{(t,0)}}\|_2^2 \Big| \calF^{t}}  
\le 
\frac{2 s^2}{m}
\pth{\sigma^2
+ 
\sum_{i=1}^m \norm{\nabla F_i(\x_i^t)}^2}.
\end{align*}
Combing the above upper bounds of (C) and (D)
and applying Proposition \ref{prop: average gradient to global gradient}, we get 
\begin{align}
\nonumber
\expect{\norm{\frac{1}{m} \bm{G}^{(t)} \bm{1}}^2 \mid \calF^{t}} 
&\le 2 s^2\sigma^2\pth{\frac{2}{ m} + \frac{\kappa^2 \eta^2 s^2 L^2}{2}} 
+ 6 s^2L^2\pth{2 +\frac{\kappa^2 \eta^2 s^2 L^2}{2}} \frac{1}{m} \sum_{i=1}^m \norm{\x_i^t - \bar{\x}^t}^2 \\
&~~~
+6 s^2\pth{\beta^2 + 1}\pth{2 +\frac{\kappa^2 \eta^2 s^2 L^2}{2}}\norm{\nabla F(\bar{\x}^t)}^2 
+ 6 s^2\zeta^2 \pth{2 +\frac{\kappa^2 \eta^2 s^2 L^2}{2}}. 
\label{eq: bound 2}
\end{align}
\paragraph{\em (iii) Putting them together.}
Combining \eqref{eq: bound 1} and \eqref{eq: bound 2}, we get
\begin{align*}
\expect{F(\bar{\x}^{t+1})  -  F(\bar{\x}^{t}) \mid \calF^{t}}  
&\le \kappa^2 \eta^2 s^2 L^2 (\zeta^2 +\sigma^2) 
- \frac{\eta s  }{2} \norm{\frac{1}{m}\nabla \bm{F}^{(t)}\bm{1}}^2\\
&~~~ + \frac{L\eta^2 }{2}6 s^2\zeta^2 \pth{2 +\frac{\kappa^2  L^2}{2}} 
 - \qth{\frac{\eta s}{2} - \frac{\eta^2 s^2}{2} - \kappa^2 \eta^2 s^2 L^2\pth{\beta^2 + 1}} \norm{\nabla F(\bar{\x}^t)}^2\\
&~~~ + \pth{\frac{s\eta L^2 }{2m} + \kappa^2 \eta^2 s^2 \frac{L^4}{m}} \sum_{i=1}^m \norm{\x_i^t - \bar{\x}^{t}}^2 \\
& ~~~ + \frac{L\eta^2}{2}6s^2L^2\pth{2 +\frac{\kappa^2 L^2}{2}} \frac{1}{m} \sum_{i=1}^m \norm{\x_i^t - \bar{\x}^t}^2 \\
& ~~~ +\frac{ L\eta^2}{2}6 s^2\pth{\beta^2 + 1}\pth{2 +\frac{\kappa^2 L^2}{2}}\norm{\nabla F(\bar{\x}^t)}^2 \\
&~~~ + \frac{ L\eta^2}{2} 2s^2\sigma^2\pth{\frac{2}{m} + \frac{\kappa^2 L^2}{2}}
+ 
\frac{\eta^2 L d \sigma_z^2 \expect{B_t^2 \mid \calF^t}}{\gamma^2 m^2}
. 
\end{align*}
Assuming that $\eta\le {1}/[{108s L (\kappa + 1)(\beta^2+1)}]$, the above displayed equation can be simplified as 
\begin{align*}
\expect{F(\bar{\x}^{t+1})  -  F(\bar{\x}^{t}) \mid \calF^{t}}
&\le -\frac{\eta s }{3} \norm{\nabla F(\bar{\x}^t)}^2 
+ \eta s\frac{L^2 }{m} \sum_{i=1}^m \norm{\x_i^t - \bar{\x}^t}^2
+ \eta^2 s^2  6L\pth{\zeta^2+\sigma^2}\pth{1+L^2\kappa^2} 
+
\frac{\eta^2 L d \sigma_z^2 \expect{B_t^2 \mid \calF^t}}{\gamma^2 m^2}
. 
\end{align*}

\subsection{\bf Proof of~\prettyref{lmm:gradient Bt}}
    We have
    \begin{align*}
        B_t^2 &\triangleq
        \max_{i \in \calA^t} \norm{\Delta_{i}^{t}}^2
        \le
        \sum_{i=1}^m
        \norm{\Delta_{i}^{t}}^2
        =
        C_t^2
        =
        \sum_{i=1}^m
        \norm{\sum_{k=\tau_i(t) + 1}^{t}
        \sum_{r=0}^{s-1} \nabla \ell_i (\x_i^{(k,r)}; \xi_i^{k})}^2
        \le
        \sum_{i=1}^m
        (t - \tau_i(t))
        \sum_{k=\tau_i(t) + 1}^{t}
        \norm{\sum_{r=0}^{s-1} \nabla \ell_i (\x_i^{(k,r)};\xi_i^{k})}^2
        .
    \end{align*}

    Taking expectation and averaging over $T$ rounds, we have
    \begin{align}
        \frac{1}{T}\sum_{t=0}^{T-1}\expect{C_t^2}
        &\le
        \frac{1}{T}
        \sum_{t=0}^{T-1}
        \sum_{i=1}^m
        \expect{
            (t - \tau_i(t))
            \sum_{k=\tau_i(t) + 1}^{t}
            \norm{\sum_{r=0}^{s-1} \nabla \ell_i^{(k,r)}}^2
        }.
        \label{eq:bt_start}
    \end{align}

\noindent\textbf{Change of summation order.}
Writing $\sum_{k=\tau_i(t) + 1}^{t} \norm{\sum_{r=0}^{s-1} \nabla \ell_i^{(k,r)}}^2= \sum_{k=0}^{t}\indc{\tau_i(t) + 1\le k} \norm{\sum_{r=0}^{s-1} \nabla \ell_i^{(k,r)}}^2$, 
and reordering the terms, 
we obtain
\begin{align}
    \eqref{eq:bt_start}
    &=
    \frac{1}{T}
    \sum_{i=1}^m
    \sum_{k=0}^{T-1}
    \Expect \Bigg[{
        \norm{\sum_{r=0}^{s-1} \nabla \ell_i^{(k,r)}}^2
        \underbrace{
        \sum_{t=k}^{T-1}
        (t - \tau_i(t) )
        \indc{\tau_i(t) + 1 \le k}
        }_{\triangleq\; S_{i,k}}
    }\Bigg].
    \label{eq:bt_swapped}
\end{align}
For every term in the sum, $\tau_i(t) + 1\le k \le t$,
so the gradient factor always involves iterating $\x_i^k$
and the staleness factor $S_{i,k}$ always depends on
rounds $t \ge k$ only.

\noindent\textbf{Conditional independence.}
The gradient factor $\norm{\sum_{r=0}^{s-1} \nabla \ell_i^{(k,r)}}^2$
depends on the \emph{data} randomness $\xi_i^k$ at round $k$,
while $S_{i,k}$ depends on the \emph{channel} randomness
through the activation indicators
$\{\indc{\tau_i(t) + 1 \le k}\}_{t \ge k}$.
Given $\calF^k$, under the standard assumption that future channel randomness and local data are mutually independent, we have
\begin{align*}
    \expect{
        \norm{\sum_{r=0}^{s-1} \nabla \ell_i^{(k,r)}}^2
        S_{i,k} ~\Big |~ \calF^k
    }
    &=
    \expect{\norm{\sum_{r=0}^{s-1} \nabla \ell_i^{(k,r)}}^2 ~\Big |~ \calF^{k}}
    \expect{S_{i,k} ~\Big |~ \calF^{k}}
\end{align*}

\noindent\textbf{Tower property.}
For each $(i, k)$, apply the tower property
conditioning on $\calF^k$, the sigma-algebra generated by
all randomness up to round $k$:
\begin{align}
    \expect{
        \norm{\sum_r \nabla \ell_i^{(k,r)}}^2
        S_{i,k}
    }
    =
    \expect{
        \expect{
            \norm{\sum_r \nabla \ell_i^{(k,r)}}^2
            S_{i,k}
        \;\Big|\; \calF^k}
    }.
    \label{eq:tower}
\end{align}

Let $\tau_{i}^{+}(k) \triangleq \min\{t > k : i \in \calA^t\}$ denote the
first round after $k$ at which uplink $i$ becomes active.
Since $\indc{\tau_i(t) + 1 \le k} = 0$ for all $t \ge \tau_{i}^{+}(k) + 1$, the
sum $S_{i,k}$ accumulates only until device $i$'s next contribution:
\begin{align}
    \notag
    S_{i,k}
    &=
    \sum_{t=k}^{\tau_{i}^{+}(k) }(t - \tau_i(t))
    \overset{(a)}{=}
    \sum_{j=0}^{\tau_{i}^{+}(k) - k} 
    \pth{j + k - \tau_i(k)}
    =
    \frac{(\tau_{i}^{+}(k) - k)(\tau_{i}^{+}(k) - k + 1)}{2}
    +
    (k - \tau_i(k)) \pth{\tau_{i}^{+}(k) - k + 1} \\
    &=
    \frac{(\tau_{i}^{+}(k) - k)^2}{2}
    +
    \frac{\tau_{i}^{+}(k) - k}{2}
    +(k - \tau_i(k)) \pth{\tau_{i}^{+}(k) - k}
    +(k - \tau_i(k)) 
    ,
\end{align}
where equality $(a)$ follows from change of variable $j = t - k$ and the fact that $\tau_i (k) = \tau_i(j + k )$ for $0\le j \le \tau_i^{+}(k) - k$ because the uplink is inactive during that period.
We know from~\prettyref{lmm: geo second moment main text} that $\expect{(\tau_{i}^{+}(k) - k)^2} \le \frac{2}{c^2}$, $\expect{k - \tau_i(k)} \le \frac{1}{c}$, and that $\expect{\tau_{i}^{+}(k) - k} \le \frac{1}{c}$, it follows that
\begin{align}
    \expect{S_{i,k} \mid \calF^{k}}
    =
    \expect{\frac{(\tau_{i}^{+}(k) - k)^2}{2}
    +
    \frac{\tau_{i}^{+}(k) - k}{2}
    +(k - \tau_i(k)) \pth{\tau_{i}^{+}(k) - k}
    +(k - \tau_i(k))\mid \calF^k}
    \le
    \frac{1}{c^2}
    +
    \frac{1}{2c}
    +
    \pth{\frac{1}{c} + 1}
    (k - \tau_i(k))
    .
    \label{eq:term_II}
\end{align}

\noindent{\bf Bounding the gradients.}
    We know that
    \begin{align*}
    &\norm{\sum_{r=0}^{s-1} \nabla \ell_i^{(k,r)}}^2 
    \le 2 \underbrace{\norm{\sum_{r=0}^{s-1} \pth{\nabla \ell_i^{(k,r)} - \nabla \ell_i^{(k,0)}}}^2}_{(\rmA)} + 2 \underbrace{\norm{s \nabla \ell_i^{(k,0)}}^2}_{(\rmB)}. 
    \end{align*}    
    \noindent For term $(\rmA)$, %
    by Lemma \ref{lmm: local step perturbation}, we have
    \begin{align*}
    \expect{\norm{
    \sum_{r=0}^{s-1} \pth{\nabla \ell_i^{(k,r)} - \nabla \ell_i^{(k,0)}}}^2 \mid \calF^{k}}
    &\le  \frac{\kappa^2 \eta^2 s^4 L^2}{4} 
    \expect{\|\nabla \ell_i^{(k,0)}\|_2^2 \mid \calF^k} \\
    &\le \frac{\kappa^2 \eta^2 s^4 L^2}{2} (\expect{\norm{\nabla \ell_i^{(k,0)} - \nabla F_i(\x_i^{k})}^2 \mid \calF^k} + \norm{\nabla F_i(\x_i^{k})}^2) \\
    & \overset{(\rmd)}{\le}  \frac{\kappa^2 \eta^2 s^4 L^2\sigma^2}{2} + \frac{\kappa^2 \eta^2 s^4 L^2}{2} \norm{\nabla F_i(\x_i^k)}^2.
    \end{align*}
    For term $(\rmB)$,  by Assumption \ref{ass: bounded variance client-wise}, we likewise have 
    \begin{align*}
    &s^2
    \expect{\|\nabla \ell_i^{(k,0)}\|_2^2 ~|~\calF^{t}}  
    \le 
    2 s^2
    \pth{\sigma^2
    + 
    \norm{\nabla F_i(\x_i^k)}^2}.
    \end{align*}
    Putting them together, we have
    \begin{align*}
        \expect{\norm{\sum_{r=0}^{s-1} \nabla \ell_i^{(k,r)}}^2 \mid \calF^k}
        &\le
        s^2 \pth{\kappa^2 \eta^2 s^2 L^2 + 2} \pth{\sigma^2 + \norm{\nabla F_i(\x_i^k)}^2}.
    \end{align*}

It follows that
\begin{align}
    \expect{
    \norm{\sum_{r=0}^{s-1} \nabla \ell_i^{(k,r)}}^2
    S_{i,k}
    \;\Big|\; \calF^k}
    \le
    \pth{\frac{1}{c^2} + \frac{1}{2 c}
    + \pth{\frac{1}{c} + 1}
    (k - \tau_i(k))
    } s^2 \pth{\kappa^2 \eta^2 s^2 L^2 + 2}
    \pth{\sigma^2 + \norm{\nabla F_i(\x_i^k)}^2}
    .
\end{align}
By~\prettyref{ass: bounded true gradient}, we have
\begin{align*}
    \expect{
    \norm{\sum_{r=0}^{s-1} \nabla \ell_i^{(k,r)}}^2
    S_{i,k}~\Big|~ \calF^k}
    \le
    \pth{\frac{1}{c^2} + \frac{1}{2 c}
    + \pth{\frac{1}{c} + 1}
    (k - \tau_i(k))
    } s^2 \pth{\kappa^2 \eta^2 s^2 L^2 + 2}
    \pth{\sigma^2 + G_{\max}^2}.
\end{align*}

Take expectation over the remaining randomness, it holds that
\begin{align*}
    \expect{\expect{
    \norm{\sum_{r=0}^{s-1} \nabla \ell_i^{(k,r)}}^2
    S_{i,k}~\Big|~ \calF^k}}
    \le
    \pth{\frac{1}{c^2} + \frac{1}{2 c}
    + \pth{\frac{1}{c} + 1}
    \expect{(k - \tau_i(k))}
    } s^2 \pth{\kappa^2 \eta^2 s^2 L^2 + 2}
    \pth{\sigma^2 + G_{\max}^2} 
    \le
    \frac{4 s^2 \pth{\kappa^2 \eta^2 s^2 L^2 + 2}
    \pth{\sigma^2 + G_{\max}^2} }{c^2}
    .
\end{align*}
Combining everything together, we have
\begin{align*}
    \frac{1}{T} 
    \sum_{t=0}^{T-1}
    \expect{B_t^2}&
    \le 
    \frac{1}{T} 
    \sum_{t=0}^{T-1}
    \expect{C_t^2} =
    \frac{1}{T}
    \sum_{i=1}^m
    \sum_{k=0}^{T-1}
    \Expect \Bigg[{
        \norm{\sum_{r=0}^{s-1} \nabla \ell_i^{(k,r)}}^2
        \sum_{t=k}^{T-1}
        (t - \tau_i(t) )
        \indc{\tau_i(t) + 1 \le k}
    }\Bigg]
    \le
    \frac{4 m s^2 \pth{\kappa^2 \eta^2 s^2 L^2 + 2}
    \pth{\sigma^2 + G_{\max}^2} }{c^2}
    \le
    \frac{12 m s^2 
    \pth{\sigma^2 + G_{\max}^2} }{c^2}
    ,
\end{align*}
where the last inequality follows from our learning rate condition.

\subsection{\bf Proof of Lemma \ref{lmm: consensus}}
The consensus error, 
which measures the distance between the averaged model over all the devices and %
local models, 
can be written in matrix form as \eqref{eq: consensus writeup},
\begin{align}
\nonumber
\frac{1}{m}\sum_{i=1}^m \norm{\bar{\x}^{t} - \x_i^t}^2 
&\triangleq \frac{1}{m}
\fnorm{\bm{X}^{(t)} \pth{\identity - \allones}}^2
= \frac{1}{m}
\fnorm{
\pth{\bm{X}^{(t-1)} - \eta \bm{G}^{(t-1)} } W^{(t-1)} \pth{\identity - \allones} 
+  
\eta \bm{Z}^{t-1}  W^{(t-1)} (\identity - \allones)}^2 \\
\notag
&=
\frac{1}{m} 
\fnorm{\eta\sum_{q=0}^{t-1} 
\pth{\bm{G}^{(q)} - \bm{Z}^{(q)}}
\pth{
\prod_{l=q}^{t-1} W^{(l)} - \allones}
}^2 \\
\label{eq: consensus writeup}
&
\le 
\frac{2}{m} 
\fnorm{\eta\sum_{q=0}^{t-1} \bm{G}^{(q)}
\pth{
\prod_{l=q}^{t-1} W^{(l)} - \allones}
}^2
+
\frac{2}{m} 
\fnorm{\eta \sum_{q=0}^{t-1}\bm{Z}^{(q)} 
\pth{
\prod_{l=q}^{t-1} W^{l} - \allones} }^2
,
\end{align}
where the last inequality follows from Jensen's inequality.
Our proof includes two steps, $(I)$ and $(II)$.

\paragraph{Bounding $(I)$.}
Define $\Delta \bm{G}^{(r)} \triangleq \bm{G}^{(r)} -  \bm{G}_0^{(r)}$ and
$A_{r,t}\triangleq \prod_{\ell=r}^t W^{(\ell)} - \allones$.
It holds that
\begin{align}
\fnorm{\eta
\sum_{q=0}^{t-1} \bm{G}^{(q)} A_{q,t-1}}^2 
&\le 3\eta^2 
    \underbrace{\fnorm{
    \sum_{q=0}^{t-1} 
    \Delta \bm{G}^{(q)}
    A_{q,t-1}
    }^2}_{(\rmA)}
    + 3\eta^2 
    \underbrace{\fnorm{\sum_{q=0}^{t-1}\pth{\bm{G}_0^{\pth{q}} - s\nabla \bm{F}^{\pth{q}}}
    A_{q,t-1}
    }}_{(\rmB)}
    \label{eq: conseneus iterative error}
    +3\eta^2 s^2 \underbrace{\fnorm{\sum_{q=0}^{t-1}\nabla \bm{F}^{\pth{q}}
    A_{q,t-1}
    }^2}_{(\rmC)}
\end{align}
where the second equality follows from the fact that all devices are initiated at the same weights.

\noindent{\em (i) Bounding $\expect{(\rmA)}$.}
The term $(\rmA)$ in Eq.\,\eqref{eq: conseneus iterative error} arises from multiple local steps. 
We have,
\begin{align}
\nonumber
\expect{(\rmA) } 
&\overset{(\rma)}{\le}
\sum_{q=0}^{t-1} {\rho}^{t-q}\expect{
\fnorm{\Delta \bm{G}^{(q)}}^2 
} 
+ 
\sum_{q=0}^{t-1} \sum_{p=0, p\neq q}^{t-1}\expect{ 
\fnorm{\Delta \bm{G}^{(p)}A_{p,t-1}}
\fnorm{\Delta \bm{G}^{(q)}A_{q,t-1}}}
\\
\nonumber
&\overset{(\rmb)}{\le}
\sum_{q=0}^{t-1} 
{\rho}^{t-q}\expect{
\fnorm{\Delta \bm{G}^{(q)}}^2 
} 
+ 
\sum_{q=0}^{t-1} \sum_{p=0, p\neq q}^{t-1} 
\frac{\sqrt{\rho}^{2t-p-q}}{2}
\expect{
{
\fnorm{
\Delta \bm{G}^{(p)}
}^2 
+
\fnorm{
\Delta \bm{G}^{(q)}
}^2 }}, 
\end{align}
where inequality $(\rma)$ follows from \eqref{lmm: spectral norm main},
inequality $(\rmb)$ holds because of Young's inequality.
Next, we bound the second term.
It follows that
\begin{align*}
    \sum_{q=0}^{t-1} \sum_{p=0, p\neq q}^{t-1}  \frac{\sqrt{\rho}^{2t-p-q}}{2}\expect{
    {
    \fnorm{
    \Delta \bm{G}^{(p)}
    }^2 
    +
    \fnorm{
    \Delta \bm{G}^{(q)}
    }^2 } } 
    &
    \le \sum_{q=0}^{t-1} \sum_{p=0}^{t-1} \frac{\sqrt{\rho}^{2t-p-q}}{2}
    \expect{
    {
    \fnorm{
    \Delta \bm{G}^{(p)}
    }^2 
    +
    \fnorm{
    \Delta \bm{G}^{(q)}
    }^2 } }
    \le \frac{\sqrt{{\rho}}}{1-\sqrt{{\rho}}}\sum_{q=0}^{t-1}  
    \sqrt{{\rho}}^{t-q}\expect{
    \fnorm{
    \Delta \bm{G}^{(q)}
    }^2 }.
\end{align*}
In addition, since ${\rho}<1$, it holds that ${\rho}^{t-q} \le \sqrt{\rho}{\rho}^{\frac{t-q}{2}}$ for any $q\le t-1$. Thus, we have 
\begin{align}
\nonumber
\expect{(\rmA)}
& \le 
\sqrt{{\rho}}\sum_{q=0}^{t-1}{\rho}^{\frac{t-q}{2}} \expect{\fnorm{
\Delta \bm{G}^{(q)}
}^2 } 
 + \frac{\sqrt{{\rho}}}{1-\sqrt{{\rho}}}\sum_{q=0}^{t-1}  \sqrt{{\rho}}^{t-q}\expect{\fnorm{
\Delta \bm{G}^{(q)}
}^2}
\label{eq: rho T1}
\le 
\frac{2\sqrt{{\rho}}}{1-\sqrt{{\rho}}}
\sum_{q=0}^{t-1}  \sqrt{{\rho}}^{t-q}\expect{\fnorm{{\bm{G}^{\pth{q}} - \bm{G}_0^{\pth{q}}}}^2} \\
&=
\frac{2\sqrt{{\rho}}}{1-\sqrt{{\rho}}}
\sum_{q=0}^{t-1}  
\sqrt{{\rho}}^{t-q}\expect{\fnorm{{\bm{G}^{\pth{q}} - \bm{G}_0^{\pth{q}}}}^2}. 
\end{align}
It remains to bound $\expect{\fnorm{\Delta \bm{G}^{\pth{q}}}^2 },$
\begin{align*}
\expect{\fnorm{\Delta \bm{G}^{\pth{q}}}^2 } 
&\overset{(\rmc)}{\le} \kappa^2 \eta^2 s^2 L^2 \expect{\fnorm{\bm{G}_0^{\pth{q}} - s\nabla \bm{F}^{\pth{q}} + s\nabla \bm{F}^{\pth{q}}  }^2 }\\
& \le 2 \kappa^2 \eta^2 s^2 L^2 \expect{\fnorm{\bm{G}_0^{\pth{q}} - s\nabla \bm{F}^{\pth{q}}}^2 } 
+ 2 \kappa^2 s^2 \eta^2 s^4 L^2 \expect{\fnorm{\nabla \bm{F}^{\pth{q}}}^2 }\\
& \le 2 \kappa^2 s^2 \eta^2 s^4 L^2 m \sigma^2 
+ 2 \kappa^2 s^2 \eta^2 s^4 L^2 \expect{\fnorm{\nabla \bm{F}^{\pth{q}}}^2 }, 
\end{align*}
where inequality ($\rmc$) follows from Lemma \ref{lmm: local step perturbation}, adding and subtracting.
Thus,
\begin{align*}
\expect{(\rmA) } 
&\le 
\frac{2\sqrt{{\rho}}}{1-\sqrt{{\rho}}}
\sum_{q=0}^{t-1} 
\sqrt{{\rho}}^{t-q}\expect{\fnorm{\bm{G}^{\pth{q}} - \bm{G}_0^{\pth{q}}}^2 }
\le \frac{4\kappa^2 s^2 \eta^2 s^4 L^2 m \sigma^2 \rho}{\pth{1-\sqrt{{\rho}}}^2} 
+ 
\frac{4\kappa^2 s^2 \eta^2 s^4 L^2 \sqrt{\rho} }{1-\sqrt{{\rho}}}
\sum_{q=0}^{t-1} \sqrt{{\rho}}^{t-q}
\expect{\fnorm{\nabla \bm{F}^{\pth{q}}}^2}.
\end{align*}
\noindent{\em (ii) Bounding $\expect{(\rmB)}$.}
\begin{align*}
\expect{(\rmB) } 
&\le 
\sum_{q=0}^{t-1}
{\rho}^{t-q}\expect{
\fnorm{\bm{G}_0^{\pth{q}} - s\nabla \bm{F}^{\pth{q}}}^2}
\le \frac{{\rho} m s^2 \sigma^2}{1-{\rho}} 
= \frac{{\rho} m s^2 \sigma^2}{\pth{1-\sqrt{{\rho}}}\pth{1 + \sqrt{ \rho}}}
\le
\frac{{\rho} m s^2 \sigma^2}{\pth{1-\sqrt{{\rho}}}^2}.
\end{align*}

\noindent{\em (iii) Bounding $\expect{(\rmC)}$.}
Use a similar derivation as in \prettyref{eq: rho T1}, and we get
\begin{align*}
\expect{(\rmC)} 
&\le
\frac{2 \sqrt{\rho}}{1 - \sqrt{\rho}}
\sum_{q=0}^{t-1} \sqrt{\rho}^{t-q}
\expect{\fnorm{\nabla \bm{F}^{\pth{q}}}^2 }.
\end{align*}
Furthermore, we have
\begin{align*}
\frac{2 \sqrt{\rho}}{1 - \sqrt{\rho}}
\sum_{t=0}^{T-1}\sum_{q=0}^{t-1} 
\sqrt{{\rho}}^{t-q}
&\expect{\fnorm{\nabla \bm{F}^{\pth{q}}}^2} 
=\frac{2 \sqrt{\rho}}{1 - \sqrt{\rho}}
\sum_{t=0}^{T-2} \expect{\fnorm{\nabla \bm{F}^{\pth{t}}}^2} \sum_{q=1}^{T-1-t}\sqrt{{\rho}}^{q} 
\le 
\frac{2 \rho}{\pth{1-\sqrt{\rho}}^2}\sum_{t=0}^{T-1}\expect{\fnorm{\nabla \bm{F}^{\pth{t}}}^2}.
\end{align*}

\paragraph{Bounding $(II)$.}
By using a similar techniques as when bounding $\expect{(A)}$, it holds that
\begin{align*}
\fnorm{\sum_{q=0}^{t-1}\bm{Z}^{(q)}  \pth{\prod_{l=q}^{t-1} W^{l} - \allones} }^2
&=
\fnorm{\sum_{q=0}^{t-1}\bm{Z}^{(q)} A_{q}^{t-1}}^2
\le
\sum_{q=0}^{t-1} 
\expect{\fnorm{\bm{Z}^{(q)} A_{q,t-1}}^2 } 
+ \sum_{q=0}^{t-1} \sum_{p=0, p\neq q}^{t-1}\expect{ 
\fnorm{
\bm{Z}^{(p)}
A_{p,t-1}
}
\fnorm{
\bm{Z}^{(q)}
A_{q,t-1}
} 
}.
\end{align*}

We know that both $W^{(q)}$ is doubly stochastic, \ie, 
$W^{(q)} \allones = \allones$.
It holds that
\begin{align*}
    A_{q, t-1}
    &=
    \prod_{l=q}^{t-1} W^{l} - \allones
    =
    W^{(q)} \pth{\prod_{l=q+1}^{t-1} W^{l} - \allones}
    =
    \pth{W^{(q)} - \allones + \allones}
    \pth{\prod_{l=q+1}^{t-1} W^{l} - \allones} \\
    &=
    \pth{W^{(q)} - \allones}
    \pth{\prod_{l=q+1}^{t-1} W^{l} - \allones}
    +
    \allones 
    \pth{\prod_{l=q+1}^{t-1} W^{l} - \allones}
    =
    \pth{W^{(q)} - \allones}
    \pth{\prod_{l=q+1}^{t-1} W^{l} - \allones}.
\end{align*}

Recall that $\bm{Z}^{(t)} = \frac{B_t}{\gamma}
\frac{\indc{\abth{\calA^t} > 0}}{{{\abth{\calA^t} + \indc{\calA^{t} = 0}}}}
\qth{\bm{z}^t \indc{1 \in \calA^t}, \cdots, \bm{z}^t \indc{m \in \calA^t}}.$
Let $\mathbb{I}_{\calA^t} \triangleq \qth{\indc{1 \in \calA^t}, \cdots, \indc{m \in \calA^t}}$.

It remains to bound $\expect{\fnorm{\bm{Z}^{(q)} A_{q,t-1}}^2}$.
Recall that $C_q^2 \triangleq \sum_{i=1}^m \norm{\Delta_{i}^{q}}^2 \ge B_q$, and that both $\bm{z}^q$ and $C_q^2$ are independent of $\calA^q$.
\begin{align*}
    \expect{\fnorm{\bm{Z}^{(q)} A_{q,t-1}}^2}
    &=
    \expect{
    \fnorm{ 
    \frac{B_q \bm{z}^q  \mathbb{I}_{\calA^q}}{\gamma}
    \frac{\indc{\abth{\calA^q}>0}}{\abth{\calA^q} + \indc{\abth{\calA^q} = 0} } A_{q, t-1}}^2}
    \le
    \frac{\expect{C_q^2}}{\gamma^2}
    \expect{\fnorm{ \frac{ \bm{z}^{q} \mathbb{I}_{\calA^t}
    \indc{\abth{\calA^q}>0}}{\abth{\calA^q} + \indc{\abth{\calA^q} = 0} } A_{q, t-1}}^2} \\
    &=
    \frac{\expect{C_q^2}}{\gamma^2}
    \expect{\fnorm{ 
    \frac{\indc{\abth{\calA^q}>0}  \bm{z}^{q} \mathbb{I}_{\calA^t}
    \pth{W^{(q)} - \allones}
    }{\abth{\calA^q} + \indc{\abth{\calA^q} = 0} }
    \pth{\prod_{\ell=q+1}^{t-1} W^{(\ell)} - \allones}}^2}
\end{align*}

For the term $ \frac{\indc{\abth{\calA^q}>0}  \bm{z}^{q} \mathbb{I}_{\calA^t}
    \pth{W^{(q)} - \allones}
    }{\abth{\calA^q} + \indc{\abth{\calA^q} = 0} }$, we have two cases:
\paragraph{Case \#1: $\calA^q \not= \emptyset$}
We have
\begin{align*}
    \fnorm{ \frac{\indc{\abth{\calA^q}>0}  \bm{z}^{q} \mathbb{I}_{\calA^t}
    \pth{W^{(q)} - \allones}
    }{\abth{\calA^q} + \indc{\abth{\calA^q} = 0} }}^2
    &=
    \fnorm{\frac{\bm{z}^{q} \mathbb{I}_{\calA^q}
    \pth{W^{(q)} - \allones}}{\abth{\calA^q}}}^2 
    =
    \frac{\norm{\bm{z}^{q}}^2}{\abth{\calA^q}^2}
    \fnorm{\mathbb{I}_{\calA^q}
    \pth{W^{(q)} - \allones}}^2  \\
    &\le
    \frac{\norm{\bm{z}^{q}}^2 \norm{\mathbb{I}_{\calA^q}}^2}{\abth{\calA^q}^2}
    \fnorm{W^{(q)} - \allones}^2
    =
    \frac{\norm{\bm{z}^{q}}^2 }{\abth{\calA^q}}
    \fnorm{W^{(q)} - \allones}^2
    \le
    \norm{\bm{z}^{q}}^2 
    \fnorm{W^{(q)} - \allones}^2
    ,
\end{align*}
where the inequality follows from $\abth{\calA^q} \ge 1$.

\paragraph{Case \#2: $\calA^q = \emptyset$}
We have 
\[
\fnorm{ \frac{\indc{\abth{\calA^q}>0}  \bm{z}^{q} \mathbb{I}_{\calA^q}
    \pth{W^{(q)} - \allones}
    }{\abth{\calA^q} + \indc{\abth{\calA^q} = 0} }}^2 = 0 .
\]
Combining both cases, we have
\begin{align*}
    \fnorm{ \frac{\indc{\abth{\calA^q}>0}  \bm{z}^{q} \mathbb{I}_{\calA^q}
    \pth{W^{(q)} - \allones}
    }{\abth{\calA^q} + \indc{\abth{\calA^q} = 0} }}^2
    =
    \fnorm{ \frac{\indc{\abth{\calA^q}>0}  \bm{z}^{q} \mathbb{I}_{\calA^q}
    \pth{W^{(q)} - \allones}
    }{\abth{\calA^q} + \indc{\abth{\calA^q} = 0} }}^2
    \pth{\indc{\calA^q \not = \emptyset} + \indc{\calA^q = \emptyset}}
    \le
    0 \indc{\calA^q = \emptyset} +
    \norm{\bm{z}^{q}}^2
    \fnorm{W^{(q)} - \allones}^2  \indc{\calA^q \not = \emptyset}
    \le
    \norm{\bm{z}^{q}}^2
    \fnorm{W^{(q)} - \allones}^2.
\end{align*}
It follows that
\begin{align*}
    \expect{\fnorm{\bm{Z}^{(q)} A_{q,t-1}}^2}
    &\le
    \frac{\expect{C_q^2}}{\gamma^2}
    \expect{\fnorm{ 
    \frac{\indc{\abth{\calA^q}>0}  \bm{z}^{q} \mathbb{I}_{\calA^q}
    \pth{W^{(q)} - \allones}
    }{\abth{\calA^q} + \indc{\abth{\calA^q} = 0} }
    \pth{\prod_{\ell=q+1}^{t-1} W^{(\ell)} - \allones}}^2}
    \le
    \frac{\expect{C_q^2}}{\gamma^2}
    \expect{\fnorm{ 
    \frac{\indc{\abth{\calA^q}>0}  \bm{z}^{q} \mathbb{I}_{\calA^q}
    \pth{W^{(q)} - \allones}
    }{\abth{\calA^q} + \indc{\abth{\calA^q} = 0} }
    }^2}
    {\rho}^{t-q-1} \\
    &\le
    \frac{\expect{C_q^2}}{\gamma^2}
     \expect{\norm{\bm{z}^{q}}^2\fnorm{W^{(q)} - \allones}^2}
    {\rho}^{t-q-1}
    \overset{(d)}{=}
    \frac{\expect{\norm{\bm{z}^{q}}^2} \expect{C_q^2}}{\gamma^2}
     \expect{\fnorm{W^{(q)} - \allones}^2}
    {\rho}^{t-q-1}
    \le
    \frac{d \sigma_z^2 \expect{C_q^2}}{\gamma^2}
    {\rho}^{t-q}
    \fnorm{\identity}^2
    =
    \frac{m d \sigma_z^2 \expect{C_q^2}}{\gamma^2}
    {\rho}^{t-q}
    ,
\end{align*}
where equality $(d)$ follows from $\bm{z}^{q}$ and $(\calA^{q}, W^{(q)})$ are independent.

For ease of notation, let $H_q = m {\rho d \sigma_z^2 \expect{C_q^2}}/{\gamma^2}$.
It follows that
\begin{align*}
\expect{\fnorm{\sum_{q=0}^{t-1}\bm{Z}^{(q)}  \pth{W^{(q)} \prod_{l=q+1}^{t-1} W^{l} - \allones} }^2}
&\le
\sum_{q=0}^{t-1} 
\expect{\fnorm{\bm{Z}^{(q)} A_{q,t-1}}^2 } 
+ \sum_{q=0}^{t-1} \sum_{p=0, p\neq q}^{t-1}\expect{ 
\fnorm{
\bm{Z}^{(p)}
A_{p,t-1}
}
\fnorm{\bm{Z}^{(q)}
A_{q,t-1}}} \\
&\le
\sum_{q=0}^{t-1} H_q {\rho}^{t-q}
+
\sum_{q=0}^{t-1} \sum_{p=0, p\neq q}^{t-1}\expect{ 
\fnorm{
\bm{Z}^{(p)}
A_{p,t-1}
}
\fnorm{\bm{Z}^{(q)}
A_{q,t-1}}} \\
&\le
\sum_{q=0}^{t-1} H_q {\rho}^{t-q}
+
\sum_{q=0}^{t-1} \sum_{p=0, p\neq q}^{t-1} 
\frac{\sqrt{{\rho}}^{2t-p-q}}{2}
 \pth{H_p + H_q}.
\end{align*}
We have
\begin{align*}
    \sum_{q=0}^{t-1} \sum_{p=0, p\neq q}^{t-1}  \frac{\sqrt{{\rho}}^{2t-p-q}}{2}\pth{H_p + H_q} 
    &
    \le \sum_{q=0}^{t-1} \sum_{p=0}^{t-1} \frac{\sqrt{{\rho}}^{2t-p-q}}{2}
    \pth{H_p + H_q}
    \le \frac{\sqrt{{\rho}}}{1-\sqrt{{\rho}}}\sum_{q=0}^{t-1}  \sqrt{{\rho}}^{t-q} H_q.
\end{align*}
In addition, since ${\rho}<1$, it holds that ${\rho}^{t-q} \le \sqrt{{\rho}}{\rho}^{\frac{t-q}{2}}$ for any $q\le t-1$. Thus, we have 
\begin{align*}
\expect{\fnorm{\bm{Z}^{(q)} A_{q,t-1}}^2}
& \le 
\sqrt{{\rho}}\sum_{q=0}^{t-1}{\rho}^{\frac{t-q}{2}} H_q 
 + \frac{\sqrt{{\rho}}}{1-\sqrt{{\rho}}}\sum_{q=0}^{t-1}  \sqrt{{\rho}}^{t-q} H_q
\le
\frac{2 \sqrt{\rho}}{1 - \sqrt{\rho}}
\sum_{q=0}^{t-1}  \sqrt{{\rho}}^{t-q} H_q. 
\end{align*}
It follows that
\begin{align*}
\sum_{t=0}^{T-1}\sum_{q=0}^{t-1} 
\sqrt{{\rho}}^{t-q}
&H_q
=\sum_{t=0}^{T-2} H_t
\sum_{q=1}^{T-1-t}\sqrt{{\rho}}^{q} 
\le 
\frac{\sqrt{{\rho}}}{\pth{1-\sqrt{{\rho}}}}\sum_{t=0}^{T-1} H_t,
\end{align*}
which leads to
\begin{align*}
\sum_{q=0}^{T-1}
\expect{\fnorm{\bm{Z}^{(q)} A_{q,t-1}}^2}
&\le \frac{2 \sqrt{{\rho}}}{(1-\sqrt{{\rho}})}
\sum_{t=0}^{T-1}
\sum_{q=0}^{t-1}  \sqrt{{\rho}}^{t-q} H_q
=
\frac{2 \sqrt{{\rho}}}{(1-\sqrt{{\rho}})}
\sum_{t=0}^{T-2}
H_t
\sum_{q=1}^{T-1-t}  
\sqrt{{\rho}}^{q} 
\le
\frac{2 \rho}{(1-\sqrt{{\rho}})^2}
\sum_{t=0}^{T-2}
H_t.
\end{align*}

\noindent{\em (iv) Putting them together.}
\begin{align}
\nonumber
\frac{1}{mT}\sum_{t=0}^{T-1}\expect{\fnorm{\bm{X}^{\pth{t}} \pth{\identity - \allones}}^2}
&\le 
6 
\eta^2 s^2 \sigma^2 
\frac{
\rho
\pth{1 + \kappa^2 \eta^2 s^2 L^2}}{\pth{1 - \sqrt{{\rho}}}^2}
+ 
\pth{\frac{\kappa^2 \eta^2 s^2 L^2}{2} + 1}
\frac{12 \eta^2 s^2 \rho 
}{mT\pth{1-\sqrt{{\rho}}}^2}\sum_{t=0}^{T-1}\expect{\fnorm{\nabla \bm{F}^{\pth{t}}}^2} 
+
\frac{4 \eta^2 \rho d \sigma_z^2 }{\gamma^2 \pth{1-\sqrt{{\rho}}}^2 T}
\sum_{t=0}^{T-1}
\expect{C_t^2}
\\
\nonumber
&\overset{(e)}{\le}
\frac{18 {\rho \eta^2 s^2} 
}{(1-\sqrt{{\rho}})^2} \frac{1}{mT} \sum_{t=0}^{T-1} \fnorm{\nabla \bm{F}^{(t)}} 
+ \frac{18 \rho \eta^2 s^2
\sigma^2}{(1-\sqrt{{\rho}})^2}
+
\frac{4 \eta^2 \rho d \sigma_z^2 }{\gamma^2 \pth{1-\sqrt{{\rho}}}^2 T}
\sum_{t=0}^{T-1}
\expect{C_t^2}
,
\end{align}
where we assume that $\eta \le \frac{1}{s \kappa L}$ in inequality $(e)$.
We know, from~\prettyref{lmm:gradient Bt}, that
\begin{align*}
    \frac{1}{T}\sum_{t=0}^{T-1}\expect{C_t^2} 
    &\le
    \frac{12 m s^2
    \pth{\sigma^2 + G_{\max}^2} }{c^2}.
\end{align*}
It follows that
\begin{align*}
    \qth{1 - 
    \frac{54}{(1-\sqrt{\rho})^2
    }}
    \frac{1}{mT}\sum_{t=0}^{T-1}\expect{\fnorm{\bm{X}^{\pth{t}} \pth{\identity - \allones}}^2}
    &\le
    \frac{54}{(1 - \sqrt{\rho})^2}
    \frac{\rho \eta^2 s^2 (\beta^2 + 1)}{T}
    \sum_{t=0}^{T-1}
    \expect{\norm{\nabla F (\bar{\x}^t)}^2} \\
    &~~~
    + \frac{18 \rho \eta^2 s^2 \sigma^2}{(1-\sqrt{\rho})^2} 
    +
    \frac{48 \rho m \eta^2 s^2 d \sigma_z^2 (\sigma^2 + G_{\max}^2) }{c^2 {\gamma^2 (1-\sqrt{\rho})^2}}
    .
\end{align*}

Choosing $\eta \le
\frac{1 - \sqrt{\rho}}{632 s L (\kappa + 1) \sqrt{\beta^2 + 1}}$, we have 
\begin{align*}
    1 - 
    \frac{54 \eta^2 s^2 L^2}{(1-\sqrt{\rho})^2} 
    &\ge
    1 - \frac{1}{2}
    = \frac{1}{2}.
\end{align*}
Rearrange the terms, we have
\begin{small}
\begin{align*}
\frac{1}{mT}\sum_{t=0}^{T-1}\expect{\fnorm{\bm{X}^{\pth{t}} \pth{\identity - \allones}}^2} 
&\le
\frac{108 \rho \eta^2 s^2 (\beta^2 + 1)}{(1 - \sqrt{\rho})^2}
    \frac{1}{T} \sum_{t=0}^{T-1}
    \expect{\norm{\nabla F (\bar{\x}^t)}^2} 
    + \frac{36 \rho \eta^2 s^2 \sigma^2}{(1-\sqrt{\rho})^2} 
    + \frac{96 \rho m \eta^2 s^2 d \sigma_z^2 (\sigma^2 + G_{\max}^2) }{c^2 {\gamma^2 (1-\sqrt{\rho})^2}}
    .
\end{align*}
\end{small}

\subsection{\bf Proof of Theorem \ref{thm: main}}

In this proof,
we combine all the above intermediate results to show the final theorem.

\paragraph{\em (a) Taking expectation over the remaining randomness and a telescoping sum.}
\begin{align*}
\frac{1}{T}\sum_{t=0}^{T-1}\expect{F(\bar{\x}^{t+1})  -  F(\bar{\x}^{t}) }
&\le 
- \frac{s\eta }{3} 
\frac{1}{T}\sum_{t=0}^{T-1}
\expect{\norm{\nabla F(\bar{\x}^t)}^2} 
+ 6 L \eta^2 s^2 \pth{\kappa^2 L^2 + 1}\pth{\sigma^2 + \zeta^2} 
+ \eta s \frac{L^2}{mT}\sum_{t=0}^{T-1}\expect{\norm{\x_i^t - \bar{\x}^t}^2} 
+ \frac{\eta^2 L d \sigma_z^2}{\gamma^2 m^2}
\expect{B_t^2}
.
\end{align*}

\paragraph{\em (b) Plugging in Lemma~\ref{lmm:gradient Bt}.}
\begin{align*}
&\frac{F^\star  -  \expect{F(\bar{\x}^{0})}}{T}
\le 
6 L \eta^2 s^2 \pth{\kappa^2 L^2 + 1}\pth{\sigma^2 + \zeta^2} 
+ \frac{\eta s L^2}{mT}\sum_{t=0}^{T-1}\expect{\norm{\x_i^t - \bar{\x}^t}^2} 
+ \frac{\eta^2 L d \sigma_z^2}{\gamma^2 m^2}
\frac{1}{T}
\sum_{t=0}^{T-1}
\expect{B_t^2} 
- \frac{s\eta }{3} 
\frac{1}{T}\sum_{t=0}^{T-1}
\expect{\norm{\nabla F(\bar{\x}^t)}^2} \\
&\le
6 L \eta^2 s^2 \pth{\kappa^2 L^2 + 1}\pth{\sigma^2 + \zeta^2} 
+ \eta s  \frac{L^2}{mT}\sum_{t=0}^{T-1}\sum_{i=1}^m\expect{\norm{\x_i^t - \bar{\x}^t}^2} 
- \frac{s\eta }{3} 
\frac{1}{T}\sum_{t=0}^{T-1}
\expect{\norm{\nabla F(\bar{\x}^t)}^2} 
+
\frac{12 \eta^2 s^2 L d \sigma_z^2}{c^2 \gamma^2 m}
\pth{\sigma^2 + G_{\max}^2}.
\end{align*}
where the last inequality follows from $\kappa^2 \eta^2 s^2 L^2 \le 1$.

\paragraph{\em (c) Plugging in Lemma~\ref{lmm: consensus}.}
\begin{align*}
\frac{F^\star  -  \expect{F(\bar{\x}^{0})}}{T}
&\le
6 L \eta^2 s^2  \pth{\kappa^2 L^2 + 1}\pth{\sigma^2 + \zeta^2} 
- \frac{s\eta }{3}
\pth{1 
- 
\frac{316 \rho \eta^2 s^2 L^2 (\beta^2 + 1) }{(1 - \sqrt{\rho})^2}
}
\frac{1}{T}\sum_{t=0}^{T-1}
\expect{\norm{\nabla F(\bar{\x}^t)}^2}  
\\
&~~~
+\frac{36 \rho \eta^3 s^3 L^2 \sigma^2}{(1-\sqrt{\rho})^2} 
+ \frac{96 \rho m \eta^3 s^3 L^2 d \sigma_z^2 }{c^2 {\gamma^2 (1-\sqrt{\rho})^2}}
(\sigma^2 + G_{\max}^2)
+\frac{12 \eta^2 s^2 L d \sigma_z^2}{c^2 \gamma^2 m}
\pth{\sigma^2 + G_{\max}^2}
.
\end{align*}

We know from 
$\eta \le
\frac{1 - \sqrt{\rho}}{632 s L (\kappa + 1) \sqrt{\beta^2 + 1}}$
that
\begin{align*}
    \frac{316 \rho \eta^2 s^2 L^2 (\beta^2 + 1) }{(1 - \sqrt{\rho})^2}
    \le
    \frac{1}{2}.
\end{align*}

It follows that
\begin{align}
\frac{F^\star  -  \expect{F(\bar{\x}^{0})}}{T}
&\le
6 L \eta^2 s^2  \pth{\kappa^2 L^2 + 1}\pth{\sigma^2 + \zeta^2} 
- \frac{s\eta }{6}
\frac{1}{T}\sum_{t=0}^{T-1}
\expect{\norm{\nabla F(\bar{\x}^t)}^2}  
+\frac{36 \rho \eta^3 s^3 L^2 \sigma^2}{(1-\sqrt{\rho})^2} 
+ \frac{96 \rho m \eta^3 s^3 L^2 d \sigma_z^2 }{c^2 {\gamma^2 (1-\sqrt{\rho})^2}}
(\sigma^2 + G_{\max}^2)
+\frac{12 \eta^2 s^2 L d \sigma_z^2}{c^2 \gamma^2 m}
\pth{\sigma^2 + G_{\max}^2}.
\label{eq: main theorem intermediate}
\end{align}

Therefore, rearrange the terms in \eqref{eq: main theorem intermediate}, it follows that
\begin{align*}
\frac{1}{T}\sum_{t=0}^{T-1}
\expect{\norm{\nabla F(\bar{\x}^t)}^2}  
&\le
\frac{6(\expect{F(\bar{\x}^{0})} - F^{\ast})}{\eta s T}
+ 36 L \eta s  \pth{\kappa^2 L^2 + 1}\pth{\sigma^2 + \zeta^2} 
+\frac{216 \rho \eta^2 s^2 L^2 \sigma^2}{(1-\sqrt{\rho})^2} 
+ \frac{576 \rho m \eta^2 s^2 L^2 d \sigma_z^2 }{c^2 {\gamma^2 (1-\sqrt{\rho})^2}}
(\sigma^2 + G_{\max}^2)
+\frac{72 \eta s L d \sigma_z^2}{c^2 \gamma^2 m} \\
&\le
\frac{6(\expect{F(\bar{\x}^{0})} - F^{\ast})}{\eta s T}
+ 36 L \eta s  \pth{\kappa^2 L^2 + 1}\pth{\sigma^2 + \zeta^2} 
+\frac{216 \rho \eta^2 s^2 L^2 \sigma^2}{(1-\sqrt{\rho})^2} 
+\frac{73 \eta s L d \sigma_z^2}{c^2 \gamma^2 m}
\pth{\sigma^2 + G_{\max}^2}.
\end{align*}
where the inequality holds because $\eta \le
\frac{1 - \sqrt{\rho}}{632 s L m (\kappa + 1) \sqrt{\beta^2 + 1}}$.

\section{Auxiliary Lemmas}
\label{app: auxiliary lemmas}

\begin{lemma}[\cite{xiang2024efficient}]
\label{lmm: geo second moment main text}
Define the last active round of the link $i$ as $\tau_i(t) \triangleq \{t^\prime \mid t^\prime < t, i \in \calA^{t^\prime}\}$.
Given~\prettyref{ass: threat model} 
where $p_i^t \ge c$, 
and $c$ is an absolute constant, we have 
\[
\Expect[t - \tau_i(t)] \le \frac{1}{c},
~\text{and}
~\Expect[(t - \tau_i(t))^2] \le \frac{2}{c^2}.
\]
\end{lemma}
\prettyref{lmm: geo second moment main text} 
captures the expected staleness of local updates amid highly uncertain and heterogeneous uplink availability. 
It can be readily checked that when $p_i^t$'s are identical, \prettyref{lmm: geo second moment main text} trivially holds with $\expect{ t - \tau_i(t)} = \frac{1}{c}$.

\begin{proposition}[\cite{su2023federated,xiang2023towards,xiang2025empowering}]
\label{prop: average gradient to global gradient}
For any $t\in[T-1]$, it holds that 
\begin{align}
\nonumber
\frac{1}{m}\sum_{i=1}^m\norm{\nabla F_i(\x_i^t)}^2
&\le \frac{3L^2}{m} \sum_{i=1}^m \norm{\x_i^t - \bar{\x}^t}^2  
+ 3\pth{\beta^2 + 1}\norm{\nabla F(\bar{\x}^t)}^2 
+ 3\zeta^2. 
\end{align}
\end{proposition}

\section{Additional Details of Numerical Experiments}
\label{app: additional exp}

\subsection{Experimental setups}
\label{app: numerical setup}

\noindent{\bf Hardware and Software Setups.}
\begin{itemize}
\item{\bf Hardware.}
The simulations are performed on a private cluster with 8 CPUs, 64 GB RAM, and one NVIDIA A30 GPU card.

\item {\bf Software.}
We code our experiments based on PyTorch 2.8.0 and Python 3.9.1.
\end{itemize}

\noindent{\bf Neural Network and Hyper-parameter Specifications.}

Table~\ref{tbl: cnn structures} specifies details of the structures of the convolutional neural network and training.
We initialize CNNs using the Kaiming initialization.
The initial local learning rate $\eta_0$ 
is searched based on the best performance after $100$ global rounds, 
over two grids $\{0.1, 0.03, 0.01, 0.003, 0.001, 0.0003\}$.

\noindent\textbf{Datasets and Data Heterogeneity.}
The MNIST dataset \cite{lecun2010mnist} contains 10 classes of images. 
In total, there are 50000 train images and 10000 test images.

\begin{minipage}{0.45\linewidth}
 \begin{figure}[H]
\centering
\includegraphics[width=.8\linewidth, trim=0.1cm 0.1cm 0 0.1cm, clip]{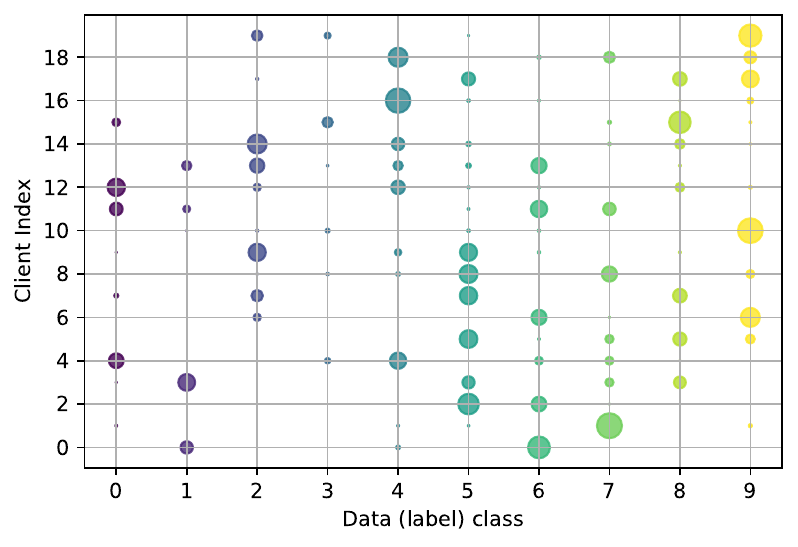}
\caption{
An example of data heterogeneity using $\mathsf{Dirichlet}(\alpha=0.1)$ distribution with $20$ clients.
$x$-axis denotes the categories of images, 
while $y$-axis denotes the client index.
The size of a circle refers to the proportion of pictures in a given class.
The color of a circle distinguishes images with different categories.
}
\centering
\label{fig: noniid 20 mnist}
\end{figure}
\end{minipage}
\hfill
\begin{minipage}{0.45\linewidth}
\begin{table}[H]
\caption{
Neural network architecture, 
loss function, 
training steps,
and batch size specifications
}
\label{tbl: cnn structures}
\begin{tabular}{cc}
\toprule
{\bf Dataset}& 
{\bf MNIST} 
\\
\toprule
Neural network &
MLP 
\\
Model architecture$^*$ & 
\begin{tabular}{p{.18\textwidth}}
\centering
{\bf L}(784,1024)
-- {\bf R}
-- {\bf L}(10)
\end{tabular}
\\
\midrule
Loss function &
Cross-entropy loss
 \\
\addlinespace[1ex]
Number of local steps $s$ &
10 \\
Number of global rounds $T$ &
800
\\
\midrule
Batch size &
32
\\
\bottomrule
\end{tabular}
\begin{tabular}{p{.95\textwidth}}
$^*$
\begin{footnotesize}    
{\bf R}: ReLU activation function;
{\bf L}: (\# outputs): a fully-connected linear layer.
\end{footnotesize}
\end{tabular}
\end{table}
\end{minipage}

\subsection{Baseline algorithms}
We compare the proposed~\fedopi~with five baseline algorithms.
(i) The ideal FedAvg algorithm \cite{mcmahan2017communication} aggregates each device's update without any receiver noise,
(ii) the SCA algorithm \cite{abrar2025non}
adopts a successive convex approximation framework to optimize the learning convergence bound based on the channel conditions,
(iii) the BB-interior algorithm \cite{zhu2019broadband}
schedules only devices within a fixed radius $R_{\text{in}}$ by excluding distant devices,
(iv) the BB-alternative algorithm \cite{zhu2019broadband}
randomly alternates between full and interior-only scheduling to balance update reliability with full data utilization,
and (v) the vanilla OTA \cite{yang2020federated} enforces channel inversion across all devices and is thus susceptible to deep-fading devices.

In terms of the CDI knowledge, the vanilla OTA, BB-interior, and BB-alternative algorithms require the parameter server to know each device's small-scale CDI $h_{i}^{t}$'s, which may incur substantial overhead.
The SCA algorithm requires the small-scale CDI at the devices and the statistical CDI at the parameter server.
In contrast, our~\fedopi~only requires the small-scale CDI at each device.
Notably, we let the SCA algorithm optimize its convergence bound based only on the {\em initial} statistical CDIs, without accounting for the potential temporal variations in our setup.

\subsection{Additional results}
\prettyref{fig:gamma tradeoff complete} plots the complete test accuracy curves of~\fedopi~reported in~\prettyref{fig:impacts of gamma}.
The observations are consistent with~\prettyref{sec: numerical experiment}, where only $\gamma = 10^{-9}$ achieves the best balance between uplink availability and receiver noise amplification.

\begin{center}
\begin{minipage}{.7\linewidth}
\begin{figure}[H]
    \centering
    \includegraphics[width=.8\linewidth]{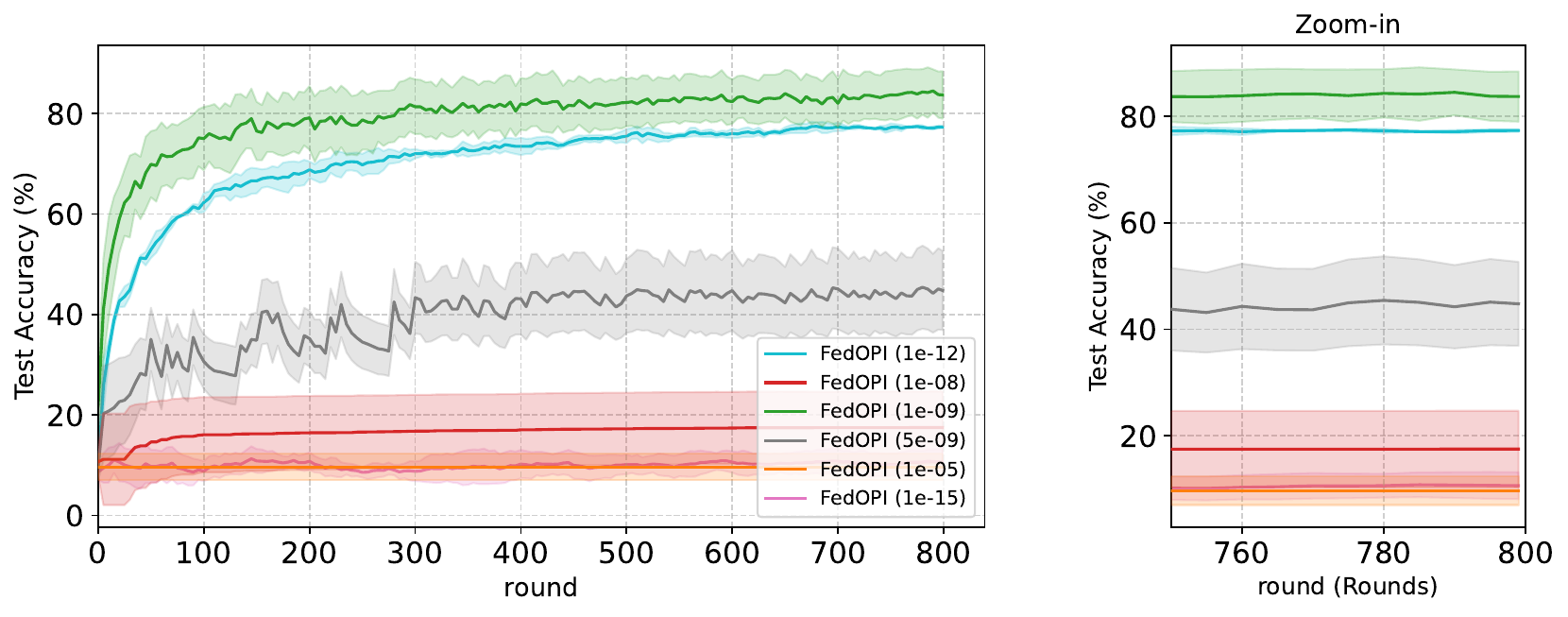}
    \caption{Comparisons of the proposed~\fedopi~under different $\gamma$'s on the MNIST dataset with $m=30$ devices in the stationary regime. 
    The results are obtained over five repetitions.
    Smaller $\gamma$ relates to lower truncation threshold and therefore leads to more active uplinks, but at the expense of higher receiver noise amplification, and vice versa.
    }
    \label{fig:gamma tradeoff complete}
\end{figure}    
\end{minipage}
\end{center}

\end{document}